\documentclass{article}

\usepackage[nonatbib, final]{neurips_2019}

\usepackage[utf8]{inputenc} 
\usepackage[T1]{fontenc}    
\usepackage{hyperref}       
\usepackage{url}            
\usepackage{booktabs}       
\usepackage{amsfonts}       
\usepackage{nicefrac}       
\usepackage{microtype}      
\usepackage{amsthm}

\usepackage{wrapfig}

\usepackage{algorithm}
\usepackage{algorithmic}
\usepackage{enumitem}
\usepackage{graphicx}

\usepackage[usenames,dvipsnames,svgnames,table]{xcolor}
\usepackage{xspace}
\usepackage{macros}

\title{Imitation-Projected Programmatic Reinforcement Learning}

\author{%
  Abhinav Verma\thanks{Equal contribution}\\
  Rice University\\
  \texttt{averma@rice.edu} \\
 \And
  Hoang M. Le$^*$ \\
  Caltech \\
  \texttt{hmle@caltech.edu} \\
  \And
  Yisong Yue \\
  Caltech \\
  \texttt{yyue@caltech.edu} \\
  \And
  Swarat Chaudhuri \\
  Rice University \\
  \texttt{swarat@rice.edu} \\
}

\begin{document}

\maketitle

\begin{abstract}
We study the problem of programmatic reinforcement learning, in which policies are represented as short programs in a symbolic language.
Programmatic policies can be more interpretable, generalizable, and amenable to formal verification than neural policies; however, designing rigorous learning approaches for such policies remains a challenge. Our approach to this challenge --- a meta-algorithm called \ippg --- is based on three insights. First, we view our learning task as optimization in policy space, modulo the constraint that the desired policy has a programmatic representation, and solve this optimization problem using a form of mirror descent 
that takes a gradient step into the unconstrained policy space and then projects back onto the constrained space.  Second, we view the unconstrained policy space as mixing neural and programmatic representations, which enables employing state-of-the-art deep policy gradient approaches.  Third, we cast the projection step as program synthesis via imitation learning, and exploit contemporary combinatorial methods for this task. We present theoretical convergence results for \ippg and empirically evaluate the approach in three continuous control domains. The experiments show that \ippg can significantly outperform state-of-the-art approaches for learning programmatic policies.
\end{abstract}

\section{Introduction}
\label{sec:intro}

A growing body of work \cite{pirl,bastani2018verifiable,zhu2019an} investigates reinforcement learning (RL) approaches that represent policies as programs in a symbolic language, e.g., a domain-specific language for composing control modules such as PID controllers \cite{ang2005pid}. Short programmatic policies offer many advantages over neural policies discovered through deep RL, including greater interpretability, better generalization to unseen environments, and greater amenability to formal verification. These benefits motivate developing effective approaches for learning such programmatic policies.

However, programmatic reinforcement learning (\pirl) remains a challenging problem, owing to the highly structured nature of the policy space.  Recent state-of-the-art approaches employ program synthesis methods to imitate or distill a pre-trained neural policy into short programs \cite{pirl,bastani2018verifiable}. However, such a distillation process can yield a highly suboptimal programmatic policy --- i.e., a large distillation gap --- and the issue of direct policy search for programmatic policies also remains open.

In this paper, we develop \ippg (Imitation-{\bf Pro}jected {\bf P}rogrammatic R{\bf e}inforcement {\bf L}earning), a new learning meta-algorithm for \pirl, as a response to this challenge.
The design of \ippg is based on three insights that enables integrating and building upon state-of-the-art approaches for policy gradients and program synthesis. First, we view programmatic policy learning as a constrained policy optimization problem, in which the desired policies are constrained to be those that have a programmatic representation. This insight motivates utilizing constrained mirror descent approaches, which take a gradient step into the unconstrained policy space and then project back onto the constrained space.
Second, by allowing the unconstrained policy space to have a mix of neural and programmatic representations, we can employ well-developed deep policy gradient approaches \cite{sutton2000policy,ddpg,schulman2015trust,schulman2017proximal,corerl} to compute the unconstrained gradient step. 
Third, we define the projection operator using program synthesis via imitation learning \cite{pirl,bastani2018verifiable}, in order to recover a programmatic policy from the unconstrained policy space.
%
Our contributions can be summarized as:
\begin{itemize}[leftmargin=*]
    \item We present \ippg, a novel meta-algorithm that is based on mirror descent, program synthesis, and imitation learning, for \pirl.
    \item On the theoretical side, we show how to cast \ippg as a form of constrained mirror descent. We provide a thorough theoretical analysis characterizing the impact of approximate gradients and projections. Further, we prove results that provide expected regret bounds and finite-sample guarantees under reasonable assumptions. 
    \item On the practical side, we provide a concrete instantiation of \ippg and evaluate it in three continuous control domains, including the challenging car-racing domain \torcs \cite{TORCS}. The experiments show significant improvements over state-of-the-art approaches for learning programmatic policies.
\end{itemize}

\section{Problem Statement}
\label{sec:problem}

The problem of programmatic reinforcement learning (\pirl) consists of a Markov Decision Process (\mdp) $\mathcal{M}$ and a programmatic policy class $\Pi$. The definition of $\mathcal{M} = (\Sc,\Ac,P,c,p_0,\gamma)$ is standard \cite{sutton2018reinforcement}, with $\Sc$ being the state space, $\Ac$ the action space, $P(s'|s,a)$ the probability density function of transitioning from a state-action pair to a new state, $c(s,a)$ the state-action cost function, $p_0(s)$ a distribution over starting states, and $\gamma \in (0,1)$ the discount factor.  A policy $\pi: \Sc \rightarrow \Ac$ (stochastically) maps states to actions. We focus on continuous control problems, so $\Sc$ and $\Ac$ are assumed to be continuous spaces.  The goal is to find a programmatic policy $\pi^*\in\Pi$ such that:
\begin{eqnarray}
\pi^* = \argmin_{\pi\in\Pi} J(\pi), \qquad\textrm{where:  } J(\pi) = \Expec\left[\sum_{i = 0}^\infty \gamma^i c(s_i,a_i\equiv\pi(s_i))\right],\label{eqn:obj}
\end{eqnarray}
with the expectation taken over the initial state distribution $s_0\sim p_0$, the policy decisions, and the transition dynamics $P$. One can also use rewards, in which case \eqref{eqn:obj} becomes a maximization problem.

\textbf{Programmatic Policy Class.}
A programmatic policy class $\Pi$ consists of policies that can be represented parsimoniously by a (domain-specific) programming language. Recent work \cite{pirl,bastani2018verifiable,zhu2019an} indicates that such policies can be easier to interpret and formally verify than neural policies, and can also be more robust to changes in the environment. 

In this paper, we consider two concrete classes of programmatic policies. The first, a simplification of the class considered in Verma et al.~\cite{pirl}, is defined by the modular, high-level language in \figref{syntax}. This language assumes a library of parameterized functions $\lop_\theta$ representing standard controllers, for instance Proportional-Integral-Derivative (PID) \cite{aastrom1984automatic} or bang-bang controllers \cite{bellman1956bang}. Programs in the language take states $s$ as inputs and produce actions $a$ as output, and can invoke fully instantiated library controllers along with predefined arithmetic, boolean and relational operators. 
The second, ``lower-level" class, from Bastani et al.~\cite{bastani2018verifiable}, consists of decision trees that map states to actions.

\begin{figure}
{\small
\begin{eqnarray*}
\pi(s) & ::= & a \mid \op(\pi_1(s),\dots, \pi_k(s)) \mid \ifc~b~\thenc~\pi_1(s)~\elsec~\pi_2(s) \mid \lop_\theta(\pi_1(s),\dots, \pi_k(s)) \\
b & ::= & \phi(s) \mid \bop(b_1,\dots, b_k)
\end{eqnarray*}
}
\vspace{-0.2in}
\caption{\textit{A high-level syntax for programmatic policies, inspired by \cite{pirl}. A policy $\pi(s)$ takes a state $s$ as input and produces an action $a$ as output. $b$ represents boolean expressions; $\phi$ is a boolean-valued operator on states; $\op$ is an operator that combines multiple policies into one policy; $\bop$ is a standard boolean operator; and $\oplus_\theta$ is a ``library function" parameterized by $\theta$.}}\figlabel{language}
\label{fig:syntax}
\end{figure}

\begin{figure}[t]
	\vskip 0.1in
{\small
	$$
	\begin{array}{l}
	\ifc~(s[\tangle] < 0.011 ~\andc~ s[\tangle] > -0.011) \smallskip \\
	\qquad \qquad ~\thenc~ \PID{\rpm,0.45,3.54,0.03,53.39}(s) ~\elsec~\PID{\rpm,0.39,3.54,0.03,53.39}(s)
	\end{array}
	$$
	}
	\vspace{-0.05in}
	\caption{\textit{A programmatic policy for acceleration in \torcs \cite{TORCS}, automatically discovered by \ippg. $s[\tangle]$ represents the most recent reading from sensor $\tangle$.
	}}
	\label{fig:code}
\end{figure}

\textbf{Example.} Consider the problem of learning a programmatic policy, in the language of \figref{syntax},
that controls a car's accelerator in the \torcs car-racing environment~\cite{TORCS}. \figref{code} shows a program in our language for this task. The program invokes PID controllers $\PID{j, j_{\mathit{target}}, \theta_P, \theta_I, \theta_D}$, where $j$ identifies the sensor (out of 29, in our experiments) that provides inputs to the controller, $j_{\mathit{target}}$ is a learned parameter that is the desired target for the value of sensor $j$, and $\theta_P$, $\theta_I$, and $\theta_D$ are respectively the real-valued coefficients of the proportional, integral, and derivative terms in the controller.  
We note that the program only uses the sensors $\tangle$ and $\rpm$. While $\tangle$ (for the position of the car relative to the track axis) is used to decide which controller to use, only the $\rpm$ sensor is needed to calculate the acceleration.

\textbf{Learning Challenges.}
Learning programmatic policies in the continuous RL setting is challenging, as the best performing methods utilize policy gradient approaches \cite{sutton2000policy,ddpg,schulman2015trust,schulman2017proximal,corerl}, but policy gradients are hard to compute in programmatic representations.  In many cases, $\Pi$ may not even be differentiable.  For our approach, we only assume access to program synthesis methods that can select a programmatic policy $\pi\in\Pi$ that minimizes imitation disagreement with demonstrations provided by a teaching oracle. 
Because imitation learning tends to be easier than general RL in long-horizon tasks \cite{sun2017deeply}, the task of imitating a neural policy with a program is, intuitively, significantly simpler than the full programmatic RL problem. 
This intuition is corroborated by past work on programmatic RL \cite{pirl}, 
which shows that direct search over programs often fails to meet basic performance objectives.

\begin{algorithm}[t]
	\caption{Imitation-Projected Programmatic Reinforcement Learning (\ippg)}
	\label{alg:ippg}
	\begin{small}
	\begin{algorithmic}[1]
		\STATE  {\bfseries Input:} Programmatic \& Neural Policy Classes: $\Pi$ \& $\F$. 
		\STATE  {\bfseries Input:} Either initial $\pi_0$ or initial $f_0$
		\STATE Define joint policy class: $\HH \equiv \Pi \oplus \F$ \ \ \ \ \ \ \ \  \comment{$h \equiv \pi + f$ defined as $h(s) = \pi(s) + f(s)$}
		\IF{given initial $f_0$}
		    \STATE $\pi_0 \leftarrow \project(f_0)$ \ \ \ \ \ \ \ \ \ \comment{program synthesis via imitation learning}
        \ENDIF
		\FOR{$t = 1, \ldots, T$}
		    \STATE $h_t \leftarrow \update_\F(\pi_{t-1}, \eta)$ \ \ \ \ \ \ \ \ \ \comment{policy gradient in neural policy space with learning rate $\eta$}
		    \STATE $\pi_t \leftarrow \project_\Pi(h_t)$ \ \ \ \ \ \ \ \ \ \comment{program synthesis via imitation learning}
		\ENDFOR
		\STATE \textbf{Return:} Policy $\pi_T$
	\end{algorithmic}
	\end{small}

\end{algorithm}

\section{Learning Algorithm}
\label{sec:algorithm}
To develop our approach, we take the viewpoint of \eqref{eqn:obj} being a constrained optimization problem, where $\Pi\subset\HH$ resides within a larger space of policies $\HH$.  In particular, we will represent $\HH \equiv \Pi \oplus \F$ using a mixing of programmatic policies $\Pi$ and neural polices $\F$.  Any mixed policy $h\equiv\pi+ f$ can be invoked as $h(s) = \pi(s) + f(s)$.  In general, we assume that $\F$ is a good approximation of $\Pi$ (i.e., for each $\pi\in \Pi$ there is some $f\in\F$ that approximates it well), which we formalize in Section \ref{sec:convergence}.

We can now frame our constrained learning problem as minimizing \eqref{eqn:obj} over $\Pi\subset\HH$,
that alternate between  taking a gradient step in the general space $\HH$ and projecting back down onto $\Pi$. This ``lift-and-project'' perspective motivates viewing our problem via the lens of mirror descent \cite{nemirovsky1983problem}.
In standard mirror descent,  the unconstrained gradient step can be written as $h \leftarrow h_{prev} - \eta \nabla_\HH J(h_{prev})$ for step size $\eta$, and the projection can be written as $\pi \leftarrow \argmin_{\pi'\in \Pi} D(\pi',h)$ for divergence measure $D$.

Our approach, \emph{Imitation-Projected Programmatic Reinforcement Learning} (\ippg), is outlined in Algorithm \ref{alg:ippg} (also see  \figref{ippg-diagram}).
\ippg is a meta-algorithm that requires instantiating two subroutines, $\update$ and $\project$, which correspond to the standard update and projection steps, respectively. \ippg can be viewed as a form of functional mirror descent with some notable deviations from vanilla mirror descent.

\begin{wrapfigure}{r}{2.1in}
\vspace{-0.15in}
\includegraphics[width=\linewidth]{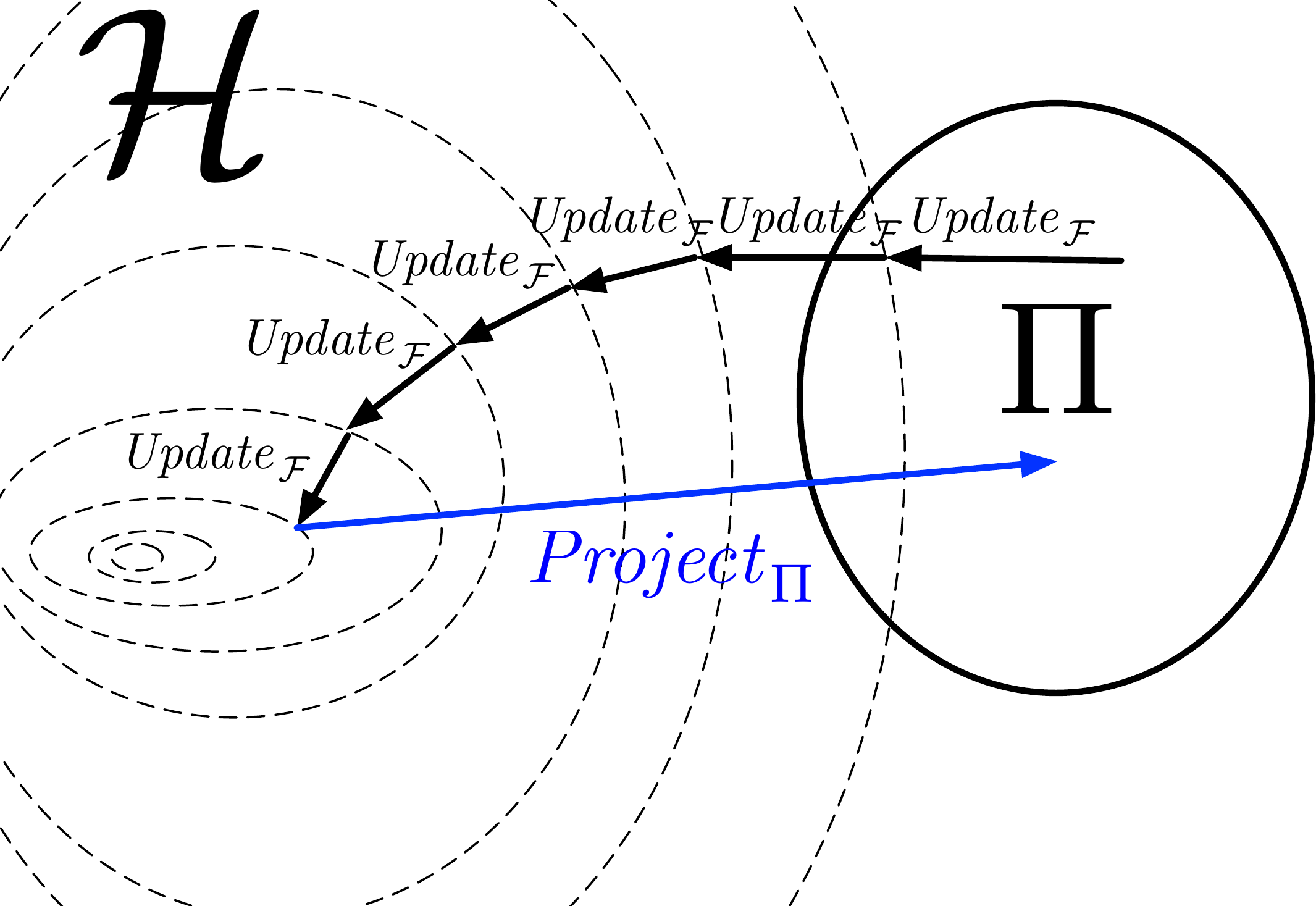}
\caption{Depicting the \ippg meta-algorithm.}\label{fig:ippg-diagram}
\vspace{-0.3in}
\end{wrapfigure}
\textbf{$\update_\F$.}
Since policy gradient methods are well-developed for neural policy classes $\F$ (e.g., \cite{ddpg,schulman2015trust,schulman2017proximal,henderson2018deep,duan2016benchmarking,corerl}) and non-existent for programmatic policy classes $\Pi$, \ippg~is designed to leverage policy gradients in $\F$ and avoid policy gradients in $\Pi$. Algorithm \ref{alg:update} shows one instantiation of $\update_\F$. Note that standard mirror descent takes unconstrained gradient steps in $\HH$ rather than $\F$, and we discuss this discrepancy between $\update_\HH$ and $\update_\F$ in Section \ref{sec:convergence}.

\textbf{$\project_\Pi$.} Projecting onto $\Pi$ can be implemented using program synthesis via imitation learning, i.e., by synthesizing a $\pi\in\Pi$ to best imitate demonstrations provided by a teaching oracle $h\in\HH$.  
Recent work  \cite{pirl,bastani2018verifiable,zhu2019an} has given 
practical heuristics for this task for various programmatic policy classes. 
%
Algorithm \ref{alg:project} shows one instantiation of $\project_\Pi$ (based on DAgger \cite{dagger}).
One complication that arises is that finite-sample runs of such imitation learning approaches only return approximate solutions and so the projection is not exact. We characterize the impact of approximate projections in Section \ref{sec:convergence}.

\textbf{Practical Considerations.} 
In practice, we often employ multiple gradient steps before taking a projection step (as also described in Algorithm \ref{alg:update}), because the step size of individual (stochastic) gradient updates can be quite small.
Another issue that arises in virtually all policy gradient approaches is that the gradient estimates can have very high variance \cite{sutton2000policy,konda2000actor,henderson2018deep}. We utilize low-variance policy gradient updates by using the reference $\pi$ as a proximal regularizer in function space \cite{corerl}.

For the projection step (Algorithm \ref{alg:project}), in practice we often retain all previous roll-outs $\tau$ from all previous projection steps. It is straightforward to query the current oracle $h$ to provide demonstrations on the states $s\in\tau$ from previous roll-outs, which can lead to substantial savings in sample complexity with regards to executing roll-outs on the environment, while not harming convergence.

\begin{algorithm}[t]
	\caption{$\update_\F$: neural policy gradient for mixed policies}
	\label{alg:update}
	\begin{small}
	\begin{algorithmic}[1]
	\STATE  {\bfseries Input:} Neural Policy Class $\F$. \qquad {\bf Input:} Reference programmatic policy: $\pi$
	\STATE  {\bfseries Input:} Step size: $\eta$. \qquad {\bf Input:} Regularization parameter: $\lambda$
	\STATE Initialize neural policy: $f_0$ \ \ \ \ \ \ \ \ \comment{any standard randomized initialization}
	\FOR{$ j = 1,\ldots,m$}
	    \STATE $f_j \leftarrow  f_{j-1} - \eta\lambda\nabla_{\F} J(\pi + \lambda f_{j-1})$ \ \ \ \ \ \ \ \ \comment{using DDPG \cite{ddpg}, TRPO \cite{schulman2015trust}, etc.,  holding $\pi$ fixed}
    \ENDFOR
    \STATE \textbf{Return:} $h \equiv \pi + \lambda f_m$
	\end{algorithmic}
	\end{small}
\end{algorithm}

\begin{algorithm}[t]
	\caption{$\project_\Pi$: program synthesis via imitation learning}
	\label{alg:project}
	\begin{small}
	\begin{algorithmic}[1]
    \STATE  {\bfseries Input:} Programmatic Policy Class: $\Pi$. \qquad {\bf Input:} Oracle policy:  $h$
    \STATE Roll-out $h$ on environment, get trajectory: $\tau_0 = (s^0,h(s^0),s^1, h(s^1),\ldots)$
    \STATE Create supervised demonstration set: $\Gamma_0 = \{(s,h(s))\}$ from $\tau_0$
    \STATE Derive $\pi_0$ from $\Gamma_0$ via program synthesis \ \ \ \ \ \ \ \ \ \comment{e.g., using methods in \cite{pirl,bastani2018verifiable}}
    \FOR{$k=1, \ldots, M$}
        \STATE Roll-out $\pi_{k-1}$, creating trajectory: $\tau_k$
        \STATE Collect demonstration data: $\Gamma' = \{(s,h(s))|s\in \tau_k\}$
        \STATE $\Gamma_k \leftarrow \Gamma' \cup \Gamma_{k-1}$  \ \ \ \ \ \ \ \ \ \comment{DAgger-style imitation learning \cite{dagger}}
        \STATE Derive $\pi_k$ from $\Gamma_k$ via program synthesis  \ \ \ \ \ \ \ \ \ \comment{e.g., using methods in \cite{pirl,bastani2018verifiable}}
    \ENDFOR
    \STATE \textbf{Return:} $\pi_M$
	\end{algorithmic}
	\end{small}
\end{algorithm}


\section{Theoretical Analysis}
\label{sec:convergence}
We start by viewing \ippg through the lens of online learning in function space, independent of the specific parametric representation. This start point yields a convergence analysis of Alg. \ref{alg:ippg} in Section \ref{sec:expected_regret} under generic approximation errors. We then analyze the issues of policy class representation in Sections \ref{sec:finite_sample} and \ref{sec:closing_gap}, and connect Algorithms \ref{alg:update} and \ref{alg:project} with the overall performance, under some simplifying conditions. 
In particular, Section \ref{sec:closing_gap} characterizes the update error in a possibly non-differentiable setting; to our knowledge, this is the first such analysis of its kind for reinforcement learning.

\textbf{Preliminaries.}
We consider $\Pi$ and $\F$ to be subspaces of an ambient policy space $\mathcal{U}$, which is a vector space equipped with inner product $\inner{\cdot,\cdot}$, induced norm $\norm{u} = \sqrt{\inner{u,u}}$,  dual norm  $\norm{v}_* = \sup\{ \inner{v,u}| \norm{u}\leq 1\}$, and standard scaling \& addition:  $(a u + b v)(s) = au(s) + bv(s)$ for $a,b \in \mathbb{R}$ and $u,v \in \mathcal{U}$. The cost functional of a policy $u$ is $J(u) = \int_\mathcal{S} c(s,u(s)) d\mu^u(s)$, where $\mu^u$ is the distribution of states induced by  $u$.
The joint policy class is $\HH = \Pi\oplus\F$, by $\HH = \{\pi + f | \forall  \pi\in\Pi, f\in\F\}$.\footnote{The operator $\oplus$ is not a direct sum, since $\Pi$ and $\F$ are not orthogonal.
} Note that $\HH$ is a subspace of $\mathcal{U}$, and inherits its vector space properties.
Without affecting the analysis, we simply equate $\mathcal{U}\equiv\HH$ for the remainder of the paper.

We assume that $J$ is convex in $\HH$, which implies that subgradient $\partial J(h)$ exists (with respect to $\HH$) \cite{bauschke2011convex}. Where $J$ is differentiable, we utilize the notion of a Fr\'echet gradient. Recall that a bounded linear operator $\nabla:\mathcal{H}\mapsto\mathcal{H}$ is called a Fr\'echet functional gradient of $J$ at $h\in\HH$ if $\lim\limits_{\norm{g}\rightarrow 0} \frac{J(h+g)-J(h) - \inner{\nabla J(h),g}}{\norm{g}} = 0$. By default, $\nabla$ (or $\nabla_\HH$ for emphasis) denotes the gradient with respect to $\HH$, whereas $\nabla_\F$ defines the gradient in the restricted subspace $\F$. 

\subsection{\ippg as (Approximate) Functional Mirror Descent}
\label{sec:expected_regret}
For our analysis, \ippg can be viewed as approximating mirror descent in (infinite-dimensional) function space over a convex set $\Pi\subset \HH$.\footnote{$\Pi$ can be convexified by considering \emph{randomized} policies, as stochastic combinations of $\pi\in\Pi$ (cf. \cite{le2019batch}).} Similar to the finite-dimensional setting \cite{nemirovsky1983problem}, we choose a strongly convex and smooth functional regularizer $R$ to be the mirror map. From the approximate mirror descent perspective, for each iteration $t$:
    \begin{enumerate}[topsep=0pt]
    \itemsep0em 
        \item Obtain a noisy gradient estimate: $\widehat{\nabla}_{t-1} \approx \nabla J(\pi_{t-1})$ 
        \item $\update_\HH(\pi)$ in $\HH$ space: $\nabla R(h_{t}) = \nabla R(\pi_{t-1}) - \eta \widehat{\nabla}_{t-1}$ \hfill \emph{(Note $\update_\HH \neq \update_\F$)}
        \item Obtain approximate projection: $\pi_t = \project_\Pi^R(h_t) \approx \argmin_{\pi\in\Pi} D_R(\pi,h_t)$
    \end{enumerate}
$D_R(u,v) = R(u) - R(v) - \inner{\nabla R(u), u-v}$ is a Bregman divergence. Taking $R(h) = \frac{1}{2}\norm{h}^2$ will recover projected functional gradient descent in $L2$-space. Here \update becomes $h_t =\pi_{t-1}-\eta\widehat{\nabla} J(\pi_{t-1})$, and \project solves for $\argmin_{\pi\in\Pi} \norm{\pi-h_t}^2$. While we mainly focus on this choice of $R$ in our experiments, note that other selections of $R$ lead to different \update and \project operators (e.g., minimizing KL divergence if $R$ is negative entropy).

The functional mirror descent scheme above may encounter two additional sources of error compared to standard mirror descent \cite{nemirovsky1983problem}. First, in the stochastic setting (also called bandit feedback \cite{flaxman2005online}), the gradient estimate $\widehat{\nabla}_t $ may be biased, in addition to having high variance. One potential source of bias is the gap between $\update_\HH$ and $\update_\F$. Second, the \project step may be inexact. We start by analyzing the behavior of \ippg under generic bias, variance, and projection errors, before discussing the implications of approximating $\update_\HH$ and $\project_\Pi$ by Algs. \ref{alg:update} \& \ref{alg:project}, respectively. Let the bias be bounded by $\beta$, i.e., $\norm{\mathbb{E}[\widehat{\nabla}_t|\pi_t] - \nabla J(\pi_t)}_*\leq \beta$ almost surely. Similarly let the variance of the gradient estimate be bounded by $\sigma^2$, and the projection error norm $\norm{\pi_t - \pi_t^*}\leq\epsilon$. We state the expected regret bound below; more details and a proof appear in Appendix \ref{sec:app_regret_bound}.
\begin{thm}[Expected regret bound under gradient estimation and projection errors]
\label{thm:expected_regret_bound}
Let $\pi_1,\ldots, \pi_T$ be a sequence of programmatic policies returned by Algorithm \ref{alg:ippg}, and $\pi^*$ be the optimal programmatic policy. Choosing learning rate $\eta = \sqrt{\frac{1}{\sigma^2}(\frac{1}{T}+\epsilon})$, we have the expected regret over $T$ iterations: 
\begin{equation}
    \label{eqn:expected_regret_bound}
\mathbb{E}\left[ \frac{1}{T}\sum_{t=1}^T J(\pi_t)\right] - J(\pi^*) = O\left( \sigma\sqrt{\frac{1}{T}+\epsilon} + \beta\right).
\end{equation}
\end{thm}
The result shows that error $\epsilon$ from \project and the bias $\beta$ do not accumulate and simply contribute an additive term on the expected regret.\footnote{Other mirror descent-style analyses, such as in \cite{sun2018dual}, lead to accumulation of errors over the rounds of learning $T$.  One key difference is that we are leveraging the assumption of convexity of $J$ in the (infinite-dimensional) function space representation.} The effect of variance of gradient estimate decreases at a $\sqrt{1/T}$ rate. Note that this regret bound is agnostic to the specific \update and \project operations, and can be applied more generically beyond the specific algorithmic choices used in our paper. 

\subsection{Finite-Sample Analysis under Vanilla Policy Gradient Update and DAgger Projection}
\label{sec:finite_sample}
Next, we show how certain instantiations of \update and \project affect the magnitude of errors and influence end-to-end learning performance from finite samples, under some simplifying assumptions on the \update step. For this analysis, we simplify Alg. \ref{alg:update} into the case $\update_\F \equiv \update_\HH$. In particular, we assume programmatic policies in $\Pi$ to be parameterized by a vector $\theta\in\mathbb{R}^k$, and $\pi$ is differentiable in $\theta$ (e.g., we can view $\Pi\subset\F$ where $\F$ is parameterized in $\mathbb{R}^k$). We further assume the trajectory roll-out is performed in an exploratory manner, where action is taken uniformly random over finite set of $A$ actions, thus enabling the bound on the bias of gradient estimates via Bernstein's inequality. The \project step is consistent with Alg. \ref{alg:project}, i.e., using DAgger \cite{ross2011reduction} under convex imitation loss, such as $\ell_2$ loss. We have the following high-probability guarantee:
\begin{thm}[Finite-sample guarantee]
\label{thm:finite_sample_main_paper}
At each iteration, we perform vanilla policy gradient estimate of $\pi$ (over $\HH$) using $m$ trajectories and, use DAgger algorithm to collect $M$ roll-outs for the imitation learning projection. Setting the learning rate $\eta = \sqrt{\frac{1}{\sigma^2}\big(\frac{1}{T} +\frac{H}{M} + \sqrt{\frac{\log(T/\delta)}{M}}\big)}$, after $T$ rounds of the algorithm, we have that:
\begin{small}
$$\frac{1}{T}\sum_{t=1}^T J(\pi_t) - J(\pi^*) \leq O\left(\sigma\sqrt{\frac{1}{T} + \frac{H}{M} +\sqrt{\frac{\log(T/\delta)}{M}}}\right) + O\left(\sigma\sqrt{\frac{\log(Tk/\delta)}{m}}+\frac{AH\log(Tk/\delta)}{m}\right) $$
\end{small}
\hspace{-3pt}holds with probability at least $1-\delta$, with $H$ being the task horizon, $A$ the cardinality of action space, $\sigma^2$ the variance of policy gradient estimates, and $k$ the dimension $\Pi$'s parameterization.
\end{thm}
The expanded result and proof are included in Appendix \ref{sec:app_finite_sample}. The proof leverages previous analysis from DAgger \cite{dagger} and the finite sample analysis of vanilla policy gradient algorithm \cite{kakade2003sample}. The finite-sample regret bound scales linearly with the standard deviation $\sigma$ of the gradient estimate, while the bias, which is the very last component of the RHS, scales linearly with the task horizon $H$. Note that the standard deviation $\sigma$ can be exponential in task horizon $H$ in the worst case \cite{kakade2003sample}, and so it is important to have practical implementation strategies to reduce the variance of the $\update$ operation. While conducted in a stylized setting, this analysis provides insight in the relative trade-offs of spending effort in obtaining more accurate projections versus more reliable gradient estimates.

\subsection{Closing the gap between $\update_\HH$ and $\update_\F$ }
\label{sec:closing_gap}
Our functional mirror descent analysis rests on taking gradients in $\HH$: $\update_\HH(\pi)$ involves  estimating $\nabla_\HH J(\pi)$ in the $\HH$ space. On the other hand, Algorithm  \ref{alg:update} performs $\update_\F(\pi)$ only in the neural policy space $\F$. In either case, although $J(\pi)$ may be differentiable in the non-parametric ambient policy space, it may not be possible to obtain a differentiable parametric  programmatic representation in $\Pi$. In this section, we discuss theoretical motivations to addressing a practical issue: \emph{How do we define and approximate the gradient $\nabla_\HH J(\pi)$ under a parametric representation?}  To our knowledge, we are the first to consider such a theoretical question for reinforcement learning.

\textbf{Defining a consistent approximation of $\nabla_\HH J(\pi)$.}
 The idea in $\update_\F(\pi)$ (Line 8 of Alg. \ref{alg:ippg}) is to approximate $\nabla_\HH J(\pi)$ by $\nabla_\F J(f)$, which has a differentiable representation, at some $f$ close to $\pi$ (under the norm). Under appropriate conditions on $\F$, we show that this approximation is valid. 
\begin{prop}
\label{prop:continuity_gradient}
Assume that (i) $J$ is Fr\'echet differentiable on $\HH$, (ii) $J$ is also differentiable on the restricted subspace $\F$, and (iii) $\F$ is dense in $\HH$ (i.e., the closure $\widebar{\F} = \HH$). Then for any fixed policy $\pi\in\Pi$, define a sequence of policies $f_k\in\F$, $k=1,2,\ldots$), that converges to $\pi$: $\lim_{k\rightarrow\infty}\norm{f_k-\pi} = 0$. We then have $\lim_{k\rightarrow\infty}\norm{\nabla_\F J(f_k) - \nabla_\HH J(\pi)}_* = 0$.
\end{prop}

Since the Fr\'echet gradient is unique in the ambient space $\HH$, $\forall k$ we have $\nabla_\HH J(f_k) = \nabla_\F J(f_k)\rightarrow \nabla_\HH J(\pi)$ as $k\rightarrow\infty$ (by Proposition \ref{prop:continuity_gradient}). We thus have an asymptotically unbiased approximation of $\nabla_\HH J(\pi)$ via differentiable space $\F$ as: $\nabla_\F J(\pi) \triangleq \nabla_\HH J(\pi) \triangleq \lim_{k\rightarrow\infty} \nabla_\F J(f_k)$.\footnote{We do not assume $J(\pi)$ to be differentiable when restricting to the policy subspace $\Pi$, i.e., $\nabla_\Pi J(\pi)$ may not exist under policy parameterization of $\Pi$.} Connecting to the result from Theorem \ref{thm:expected_regret_bound}, let $\sigma^2$ be an upper bound on the policy gradient estimates in the \emph{neural policy class $\F$}, under an asymptotically unbiased approximation of $\nabla_\HH J(\pi)$, the expected regret bound becomes $\mathbb{E}\left[ \frac{1}{T}\sum_{t=1}^T J(\pi_t)\right] - J(\pi^*) = O\left( \sigma\sqrt{\frac{1}{T}+\epsilon}\right)$.

\textbf{Bias-variance considerations of $\update_\F(\pi)$} To further theoretically motivate a practical strategy for $\update_\F(\pi)$ in Algorithm \ref{alg:update}, we utilize an equivalent proximal perspective of mirror descent \cite{beck2003mirror}, where $\update_\HH(\pi)$ is equivalent to solving for $h^\prime = \argmin_{h\in\HH} \eta\inner{\nabla_\HH J(\pi),h} + D_R(h, \pi)$. 
\begin{prop}[Minimizing a relaxed objective]
\label{prop:relaxed_objective}
For a fixed programmatic policy $\pi$, with sufficiently small constant $\lambda\in (0,1)$, we have that
\begin{equation}
\min_{h\in\HH} \eta\inner{\nabla_\HH J(\pi),h )}+D_R(h,\pi) \leq \min_{f\in\F} J\big(\pi + \lambda f \big) - J(\pi) + \inner{\nabla J(\pi),\pi}    \label{eqn:analysis_1}
\end{equation}
\end{prop}
Thus, a relaxed $\update_\HH$ step is obtained by minimizing the RHS of  \eqref{eqn:analysis_1}, i.e., minimizing $J(\pi+ \lambda f)$ over $f\in\F$. Each gradient descent update step is now $f^\prime = f - \eta \lambda \nabla_\F J(\pi_t + \lambda f)$, corresponding to Line 5 of Algorithm \ref{alg:update}. 
For fixed $\pi$ and small $\lambda$, this relaxed optimization problem becomes regularized policy optimization over $\F$, which is significantly easier. Functional regularization in policy space around a fixed prior controller $\pi$ has demonstrated significant reduction in the variance of gradient estimate \cite{corerl}, at the expense of some bias. The below expected regret bound summarizes the impact of this increased bias and reduced variance, with details included in Appendix \ref{sec:app_theoretical_motivation}.
\begin{prop}[Bias-variance characterization of $\update_\F$]
\label{prop:bias_variance_characterization}
Assuming $J(h)$ is $L$-strongly smooth over $\HH$, i.e., $\nabla_\HH J(h)$ is $L$-Lipschitz continuous, approximating $\update_\HH$ by $\update_F$ per Alg. \ref{alg:update} leads to the expected regret bound: $\mathbb{E}\left[ \frac{1}{T}\sum_{t=1}^T J(\pi_t)\right] - J(\pi^*) = O\left( \lambda\sigma\sqrt{\frac{1}{T}+\epsilon} + \lambda^2 L^2\right)$.
\end{prop}
Compared to the idealized unbiased approximation in Proposition \ref{prop:continuity_gradient}, the introduced bias here is related to the inherent smoothness property of cost functional $J(h)$ over the joint policy class $\HH$, i.e., how close $J(\pi+\lambda f)$ is to its linear under-approximation $J(\pi)+\inner{\nabla_\HH J(\pi),\lambda f}$ around $\pi$.

\section{Experiments}
\label{sec:evaluation}

We demonstrate the effectiveness of  \ippg in synthesizing programmatic controllers in three continuous control environments. For brevity and focus, this section primarily focuses on \torcs\footnote{The code for the \torcs experiments can be found at: \href{https://bitbucket.org/averma8053/propel}{https://bitbucket.org/averma8053/propel}}, a challenging race car simulator environment \cite{TORCS}. Empirical results on two additional classic control tasks, Mountain-Car and Pendulum, are provided in Appendix~\ref{sec:appendix-experiments}; those results follow similar trends as the ones described for \torcs below, and further validate the convergence analysis of \ippg.

\begin{wrapfigure}{r}{6cm}
\vspace{-0.1in}
        \centering
        \includegraphics[width=6cm]{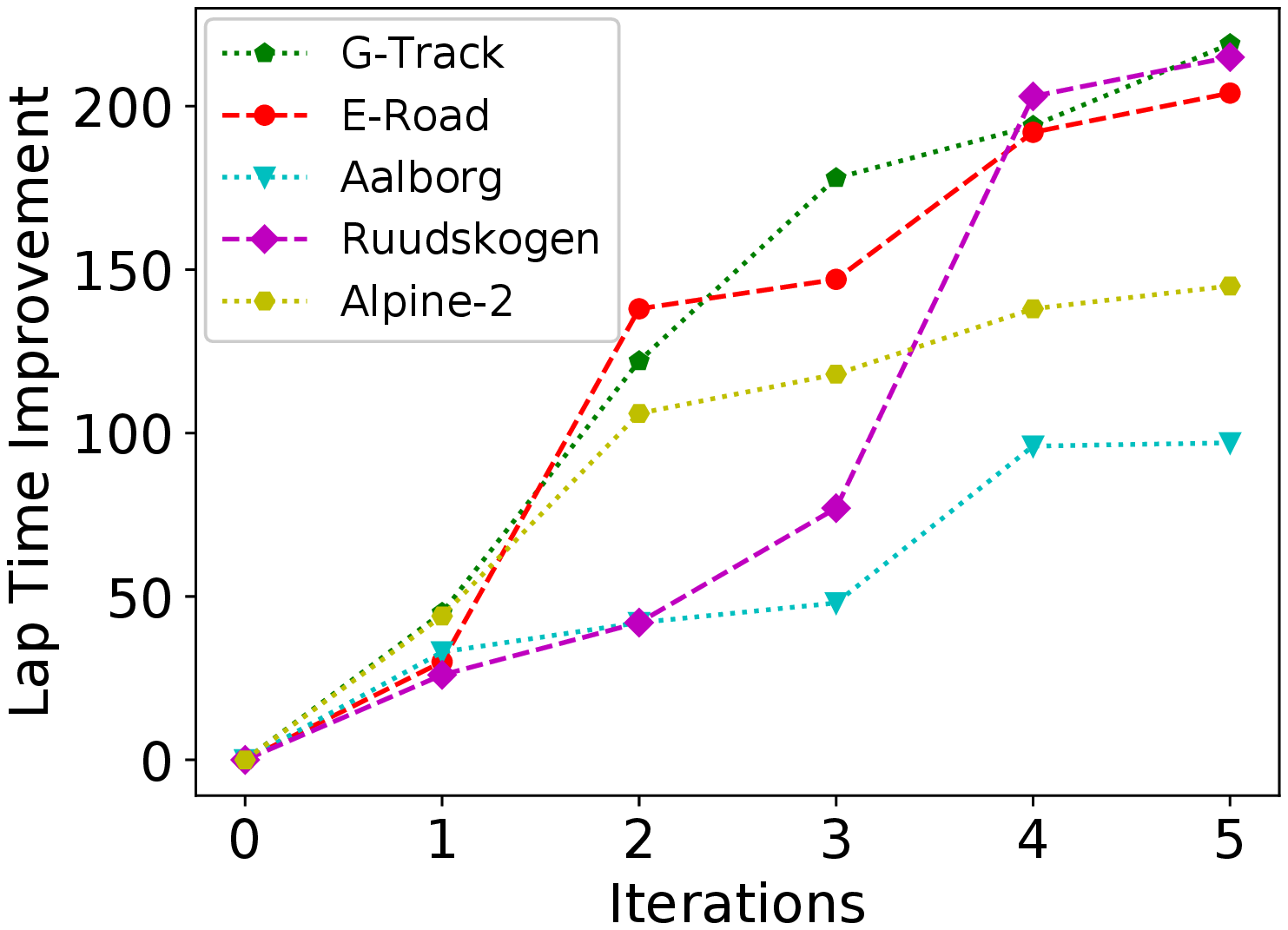}
        \vspace{-25pt}
        \caption{Median lap-time improvements during multiple iterations of \prog over $25$ random seeds.}
        \label{fig:improvements}
\vspace{-0.2in}
 \end{wrapfigure}  
 
\textbf{Experimental Setup.} We evaluate over five distinct tracks in the \torcs simulator. The difficulty of a track can be characterized by three properties; track length, track width, and number of turns. Our suite of tracks provides environments with varying levels of difficulty for the learning algorithm. The performance of a policy in the \torcs simulator is measured by the \textit{lap time} achieved on the track. To calculate the lap time, the policies are allowed to complete a three-lap race, and we record the best lap time during this race. We perform the experiments with twenty-five random seeds and report the median lap time over these twenty-five trials. Some of the policies crash the car before completing a lap on certain tracks, even after training for $600$ episodes. Such crashes are recorded as a lap time of infinity while calculating the median. If the policy crashes for more than half the seeds, this is reported as \textsc{Cr} in Tables \ref{table:performance} \& \ref{table:generalization}. We choose to report the median because taking the crash timing as infinity, or an arbitrarily large constant, heavily skews other common measures such as the mean.

 \begin{wrapfigure}{r}{6cm}
 \vspace{-0.15in}
        \centering
        \includegraphics[width=6cm]{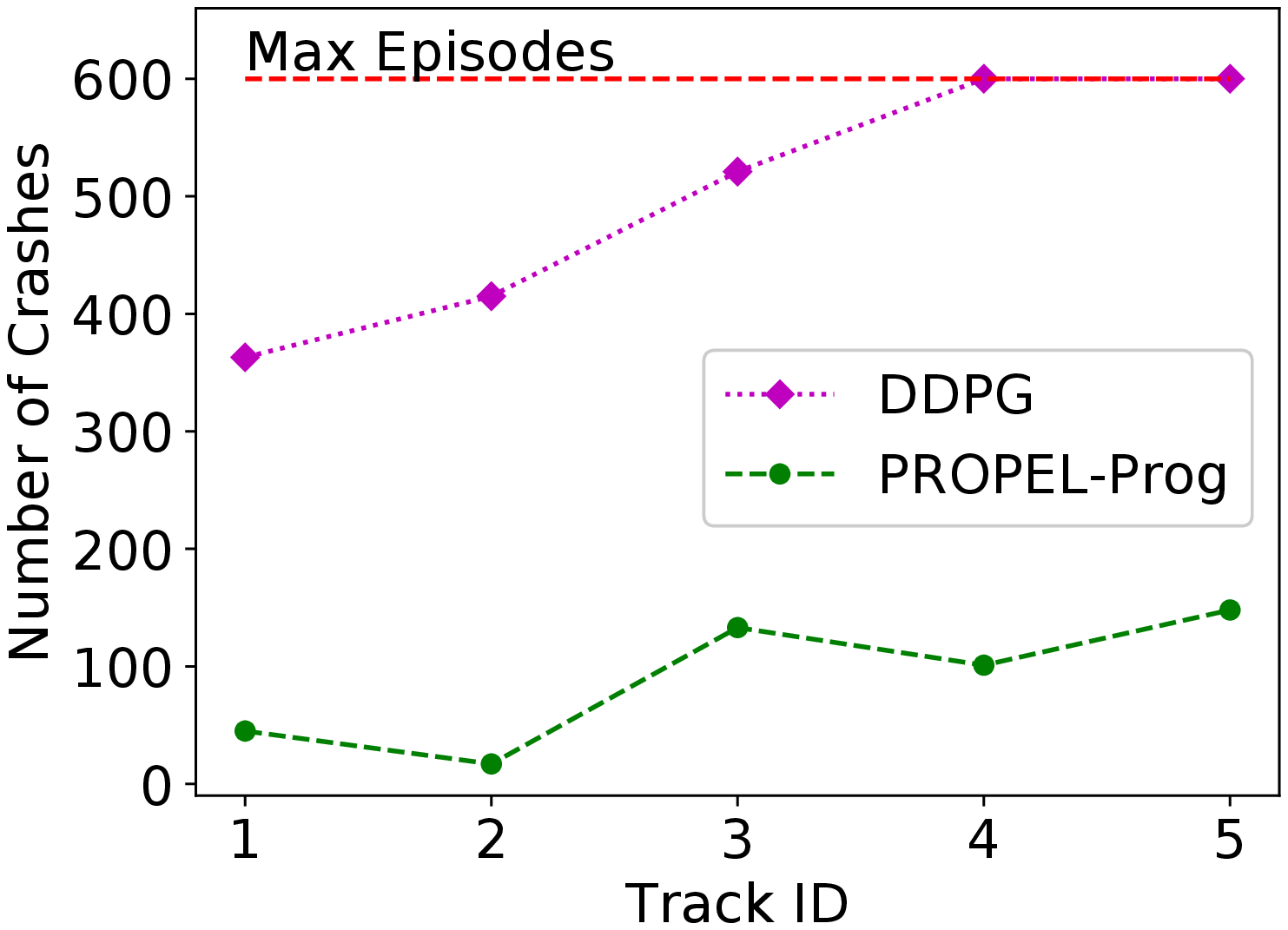}
        \vspace{-25pt}
        \caption{Median number of crashes during training of \ddpg and \prog over $25$ random seeds.}
        \label{fig:crashes}
    \vspace{-0.1in}
\end{wrapfigure}

\textbf{Baselines.} Among recent state-of-the-art approaches to learning programmatic policies are \ndps \cite{pirl} for high-level language policies, and \textsc{Viper} \cite{bastani2018verifiable} for learning  tree-based policies. Both \ndps and \textsc{Viper} rely on imitating a fixed (pre-trained) neural policy oracle, and can be viewed as degenerate versions of \ippg that only run Lines 4-6 in Algorithm \ref{alg:ippg}. We  present two \ippg analogues to \ndps and \textsc{Viper}: (i) \prog: \ippg using the high-level language of \figref{syntax} as the class of programmatic policies, similar to \ndps. (ii) \tree: \ippg using regression trees, similar to \textsc{Viper}. We also report results for \textsc{Prior}, which is a (sub-optimal) PID controller that is also used as the initial policy in \ippg. In addition, to study generalization ability as well as safety behavior during training, we also include \ddpg, a neural policy learned using the Deep Deterministic Policy Gradients~\cite{ddpg} algorithm, with $600$ episodes of training for each track. In principle, \ippg and its analysis can accommodate different policy gradient subroutines. However, in the \torcs domain, other policy gradient algorithms such as PPO and TRPO failed to learn policies that are able to complete the considered tracks. We thus focus on \ddpg as our main policy gradient component. 

 \begin{table}[t]
	\caption{\textit{Performance results in \torcs over 25 random seeds. Each entry is formatted as Lap-time / Crash-ratio, reporting median lap time in seconds over all the seeds (lower is better) and ratio of seeds that result in crashes (lower is better). A lap time of \textsc{Cr} indicates the agent crashed and could not complete a lap for more than half the seeds.}}
	\label{table:performance}
	\begin{center}
		\begin{small}
			\begin{sc}
				\begin{tabular}{l c c c c c}
					\toprule
					  & G-Track  & E-Road & Aalborg & Ruudskogen & Alpine-2  \\
					  Length & 3186m & 3260m & 2588m & 3274m & 3774m \\
					\midrule					
					Prior  & 312.92 / 0.0 & 322.59 / 0.0 & 244.19 / 0.0 & 340.29 / 0.0 & 402.89 / 0.0 \\
					Ddpg & 78.82 / 0.24 & 89.71 / 0.28 & 101.06 / 0.40 & Cr / 0.68 & Cr / 0.92 \\
					Ndps   & 108.25 / 0.24 & 126.80 / 0.28 & 163.25 / 0.40 & Cr / 0.68 & Cr / 0.92 \\
    				Viper & 83.60 / 0.24 & 87.53 / 0.28 & 110.57 / 0.40 & Cr / 0.68 & Cr / 0.92 \\
					\prog & 93.67 / 0.04 & 119.17 / 0.04 & 147.28 / 0.12 & 124.58 / 0.16 & 256.59 / 0.16 \\	    				
					\tree & 78.33 / 0.04 & 79.39 / 0.04 & 109.83 / 0.16 & 118.80 / 0.24 & 236.01 / 0.36\\

					\bottomrule
				\end{tabular}
			\end{sc}
		\end{small}
	\end{center}
\end{table}

\textbf{Evaluating Performance.} Table~\ref{table:performance} shows the performance on the considered \torcs tracks. We see that \prog~and \tree~consistently outperform the \ndps \cite{pirl} and \textsc{Viper} \cite{bastani2018verifiable} baselines, respectively.  While \ddpg outperforms \ippg on some tracks, its volatility causes it to be unable to learn in some environments, and hence to crash the majority of the time. Figure~\ref{fig:improvements} shows the consistent improvements made over the prior by \prog, over the iterations of the \ippg algorithm.  Appendix~\ref{sec:appendix-experiments} contains similar results achieved on the two classic control tasks, MountainCar and Pendulum. Figure~\ref{fig:crashes} shows that, compared to \ddpg, our approach suffers far fewer crashes while training in \torcs.

\textbf{Evaluating Generalization.} To compare the ability of the controllers to perform on tracks  not seen during training, we executed the learned policies on all the other tracks (Table~\ref{table:generalization}). We observe that \ddpg crashes significantly more often than \prog. This demonstrates the generalizability of the policies returned by \ippg. Generalization results for the \tree policy are given in the appendix. In general, \tree policies are more generalizable than \ddpg but less than \prog. On an absolute level, the generalization ability of \ippg~still leaves much room for improvement, which is an interesting direction for future work.

\textbf{Verifiability of Policies.}
As shown in prior work~\cite{bastani2018verifiable,pirl}, parsimonious programmatic policies are more amenable to formal verification than neural policies. Unsurprisingly, the policies generated by \tree and \prog are easier to verify than \ddpg policies. As a concrete example, we verified a smoothness property of the \prog policy using the {\sc Z3} SMT-solver~\cite{z3} (more details in Appendix~\ref{sec:appendix-experiments}). The verification terminated in $0.49$~seconds.

\textbf{Initialization.}
In principle, \ippg can be initialized with a random program, or a random policy trained using \ddpg. In practice, the performance of \ippg depends to a certain degree on the stability of the policy gradient procedure, which is \ddpg in our experiments. Unfortunately, \ddpg often exhibits high variance across trials and fares poorly in challenging RL domains. Specifically, in our \torcs experiments, \ddpg fails on a number of tracks (similar phenomena have been reported in previous work that experiments on similar continuous control domains \cite{henderson2018deep, corerl, pirl}). Agents obtained by initializing \ippg with neural policies obtained via \ddpg also fail on multiple tracks. Their performance over the five tracks is reported in Appendix~\ref{sec:appendix-experiments}. In contrast, \ippg  can often finish the challenging tracks when initialized with a very simple hand-crafted programmatic prior.

\begin{table}[t]
	\caption{\textit{Generalization results in \torcs, where rows are training and columns are testing tracks. Each entry is formatted as \prog~/ DDPG, and the number reported is the median lap time in seconds over all the seeds (lower is better).  \textsc{Cr} indicates the agent crashed and could not complete a lap for more than half the seeds.}}
	\label{table:generalization}
	\begin{center}
	\begin{sc}
		\begin{footnotesize}
				\begin{tabular}{lccccc}
					\toprule
					 & G-Track  & E-Road & Aalborg & Ruudskogen & Alpine-2  \\
					\midrule					
					G-Track  & -  & 124 / Cr & Cr / Cr & Cr / Cr &  Cr / Cr    \\
					E-Road  &  102 / 92 & -  & Cr / Cr & Cr / Cr & Cr / Cr \\
					Aalborg  & 201 / 91 & 228 / Cr & -  & 217 / Cr & Cr / Cr\\
					Ruudskogen  & 131 / Cr & 135 / Cr & Cr / Cr & -  & Cr / Cr\\
					Alpine-2  & 222 / Cr & 231 / Cr & 184 / Cr &  Cr / Cr & -  \\
					\bottomrule
				\end{tabular}
		\end{footnotesize}
		\end{sc}
	\end{center}
\end{table}

\section{Related Work}
\label{sec:related}

\textbf{Program Synthesis.} Program synthesis is the problem of automatically searching for a program within a language that fits a given specification~\cite{gulwani-survey}. Recent approaches to the problem have leveraged symbolic knowledge about program structure~\cite{FeserCD15}, satisfiability solvers~\cite{armandosketching,jha2010oracle}, and meta-learning techniques~\cite{Bayou,parisotto2016neuro,devlin2017robustfill,deepcoder} to generate interesting programs in many domains~\cite{sygus,flashmeta,AlurRU17}. In most prior work, the specification is a logical constraint on the input/output behavior of the target program. However, there is also a growing body of work that considers program synthesis modulo optimality objectives~\cite{BloemCHJ09,CCS14,RaychevBVK16}, often motivated by machine learning tasks~\cite{Bayou,valkov2018houdini,ellis2018learning,DuIPSSRSM18,pirl,bastani2018verifiable,zhu2019an}. Synthesis of programs that imitates an oracle has been considered in both the logical~\cite{jha2010oracle} and the optimization~\cite{pirl,bastani2018verifiable,zhu2019an} settings. 
The projection step in \ippg builds on this prior work. While our current implementation of this step is entirely symbolic, in principle, the operation can also utilize contemporary techniques for learning policies that guide the synthesis process~\cite{Bayou,deepcoder,SiYDNS19}. 

\textbf{Constrained Policy Learning.}
Constrained policy learning has seen increased interest in recent years, largely due to the desire to impose side guarantees such as stability and safety on the policy's behavior.  Broadly, there are two approaches to imposing constraints: specifying constraints as an additional cost function \cite{achiam2017constrained,le2019batch}, and explicitly encoding constraints into the policy class \cite{alshiekh2018safe,simile,corerl,dalal2018safe,berkenkamp2017safe}.  In some cases, these two approaches can be viewed as duals of each other. For instance, recent work that uses control-theoretic policies as a functional regularizer \cite{simile,corerl} can be viewed from the perspective of both regularization (additional cost) and an explicitly constrained policy class (a specific mix of neural and control-theoretic policies).  We build upon this perspective to develop the gradient update step in our approach.

\textbf{RL using Imitation Learning.} 
There are two ways to utilize imitation learning subroutines within RL.  First, one can leverage limited-access or sub-optimal experts to speed up learning \cite{ross2014reinforcement,cheng2019fast,chang2015learning,sun2018truncated}. Second, one can learn over two policy classes (or one policy and one model class) to achieve accelerated learning compared to using only one policy class \cite{montgomery2016guided,cheng2019accelerating,sun2018dual,cheng2019predictor}.  Our approach has some stylistic similarities to previous efforts \cite{montgomery2016guided,sun2018dual} that use a richer policy space to search for improvements before re-training the primary policy to imitate the richer policy.  One key difference is that our primary policy is programmatic and potentially non-differentiable.
A second key difference is that our theoretical framework takes a functional gradient descent perspective --- it would be interesting to carefully compare with previous analysis techniques to find a unifying framework.

\textbf{RL with Mirror Descent.}
The mirror descent framework has previously used to analyze and design RL algorithms. For example, Thomas et al.~\cite{thomas2013projected} and Mahadevan and Liu~\cite{mahadevan2012sparse} use composite objective mirror descent, or \textsc{Comid} \cite{duchi2010composite}, which allows incorporating adaptive regularizers into gradient updates, thus offering connections to either natural gradient RL \cite{thomas2013projected} or sparsity inducing RL algorithms \cite{mahadevan2012sparse}. Unlike in our work, these prior approaches perform projection into the same native, differentiable representation. Also, the analyses in these papers do not consider errors introduced by hybrid representations and approximate projection operators. However, one can potentially extend our approach with versions of mirror descent, e.g., \textsc{Comid}, that were considered in these efforts. 

\section{Conclusion and Future Work}
\label{sec:conclusion}

We have presented \ippg, a meta-algorithm based on mirror descent, program synthesis, and imitation learning, for programmatic reinforcement learning (\pirl). We have presented theoretical convergence results for \ippg, developing novel analyses to characterize approximate projections and biased gradients within the mirror descent framework.
We also validated \ippg empirically, and show that it can discover interpretable, verifiable, generalizable, performant policies and significantly outperform the state of the art in \pirl.


The central idea of \ippg is the use of imitation learning and combinatorial methods in implementing a projection operation for mirror descent, with the goal of optimization in a functional space that lacks gradients. While we have developed \ippg in an RL setting, this idea is not restricted to RL or even sequential decision making. 
Future work will seek to exploit this insight in other machine learning and program synthesis settings.

\begin{small}
\textbf{Acknowledgements.} This work was supported in part by United States Air Force Contract \# FA8750-19-C-0092, NSF  Award \# 1645832, NSF Award \# CCF-1704883, the Okawa Foundation, Raytheon, PIMCO, and Intel.
\end{small}


\begin{thebibliography}{10}

\bibitem{achiam2017constrained}
Joshua Achiam, David Held, Aviv Tamar, and Pieter Abbeel.
\newblock Constrained policy optimization.
\newblock In {\em Proceedings of the 34th International Conference on Machine
  Learning-Volume 70}, pages 22--31. JMLR. org, 2017.

\bibitem{alshiekh2018safe}
Mohammed Alshiekh, Roderick Bloem, R{\"u}diger Ehlers, Bettina K{\"o}nighofer,
  Scott Niekum, and Ufuk Topcu.
\newblock Safe reinforcement learning via shielding.
\newblock In {\em Thirty-Second AAAI Conference on Artificial Intelligence},
  2018.

\bibitem{sygus}
Rajeev Alur, Rastislav Bod{\'{\i}}k, Eric Dallal, Dana Fisman, Pranav Garg,
  Garvit Juniwal, Hadas Kress{-}Gazit, P.~Madhusudan, Milo M.~K. Martin, Mukund
  Raghothaman, Shambwaditya Saha, Sanjit~A. Seshia, Rishabh Singh, Armando
  Solar{-}Lezama, Emina Torlak, and Abhishek Udupa.
\newblock Syntax-guided synthesis.
\newblock In {\em Dependable Software Systems Engineering}, pages 1--25. 2015.

\bibitem{AlurRU17}
Rajeev Alur, Arjun Radhakrishna, and Abhishek Udupa.
\newblock Scaling enumerative program synthesis via divide and conquer.
\newblock In {\em Tools and Algorithms for the Construction and Analysis of
  Systems - 23rd International Conference, {TACAS} 2017, Held as Part of the
  European Joint Conferences on Theory and Practice of Software, {ETAPS} 2017,
  Uppsala, Sweden, April 22-29, 2017, Proceedings, Part {I}}, pages 319--336,
  2017.

\bibitem{ang2005pid}
Kiam~Heong Ang, Gregory Chong, and Yun Li.
\newblock Pid control system analysis, design, and technology.
\newblock {\em IEEE transactions on control systems technology},
  13(4):559--576, 2005.

\bibitem{aastrom1984automatic}
Karl~Johan {\AA}str{\"o}m and Tore H{\"a}gglund.
\newblock Automatic tuning of simple regulators with specifications on phase
  and amplitude margins.
\newblock {\em Automatica}, 20(5):645--651, 1984.

\bibitem{deepcoder}
Matej Balog, Alexander~L. Gaunt, Marc Brockschmidt, Sebastian Nowozin, and
  Daniel Tarlow.
\newblock Deepcoder: Learning to write programs.
\newblock In {\em 5th International Conference on Learning Representations,
  {ICLR} 2017, Toulon, France, April 24-26, 2017, Conference Track
  Proceedings}, 2017.

\bibitem{bastani2018verifiable}
Osbert Bastani, Yewen Pu, and Armando Solar-Lezama.
\newblock Verifiable reinforcement learning via policy extraction.
\newblock In {\em Advances in Neural Information Processing Systems}, pages
  2494--2504, 2018.

\bibitem{bauschke2011convex}
Heinz~H Bauschke, Patrick~L Combettes, et~al.
\newblock {\em Convex analysis and monotone operator theory in Hilbert spaces},
  volume 408.
\newblock Springer, 2011.

\bibitem{beck2003mirror}
Amir Beck and Marc Teboulle.
\newblock Mirror descent and nonlinear projected subgradient methods for convex
  optimization.
\newblock {\em Operations Research Letters}, 31(3):167--175, 2003.

\bibitem{bellman1956bang}
Richard Bellman, Irving Glicksberg, and Oliver Gross.
\newblock On the ``bang-bang'' control problem.
\newblock {\em Quarterly of Applied Mathematics}, 14(1):11--18, 1956.

\bibitem{berkenkamp2017safe}
Felix Berkenkamp, Matteo Turchetta, Angela Schoellig, and Andreas Krause.
\newblock Safe model-based reinforcement learning with stability guarantees.
\newblock In {\em Advances in neural information processing systems}, pages
  908--918, 2017.

\bibitem{BloemCHJ09}
Roderick Bloem, Krishnendu Chatterjee, Thomas~A. Henzinger, and Barbara
  Jobstmann.
\newblock Better quality in synthesis through quantitative objectives.
\newblock In {\em Computer Aided Verification, 21st International Conference,
  {CAV} 2009, Grenoble, France, June 26 - July 2, 2009. Proceedings}, pages
  140--156, 2009.

\bibitem{chang2015learning}
Kai-Wei Chang, Akshay Krishnamurthy, Alekh Agarwal, Hal Daum{\'e}~III, and John
  Langford.
\newblock Learning to search better than your teacher.
\newblock In {\em International Conference on Machine Learning (ICML)}, 2015.

\bibitem{CCS14}
Swarat Chaudhuri, Martin Clochard, and Armando Solar-Lezama.
\newblock Bridging boolean and quantitative synthesis using smoothed proof
  search.
\newblock In {\em POPL}, pages 207--220, 2014.

\bibitem{cheng2019predictor}
Ching-An Cheng, Xinyan Yan, Nathan Ratliff, and Byron Boots.
\newblock Predictor-corrector policy optimization.
\newblock In {\em International Conference on Machine Learning (ICML)}, 2019.

\bibitem{cheng2019accelerating}
Ching-An Cheng, Xinyan Yan, Evangelos Theodorou, and Byron Boots.
\newblock Accelerating imitation learning with predictive models.
\newblock In {\em International Conference on Artificial Intelligence and
  Statistics (AISTATS)}, 2019.

\bibitem{cheng2019fast}
Ching-An Cheng, Xinyan Yan, Nolan Wagener, and Byron Boots.
\newblock Fast policy learning through imitation and reinforcement.
\newblock In {\em Uncertainty in artificial intelligence}, 2019.

\bibitem{corerl}
Richard Cheng, Abhinav Verma, Gabor Orosz, Swarat Chaudhuri, Yisong Yue, and
  Joel Burdick.
\newblock Control regularization for reduced variance reinforcement learning.
\newblock In {\em International Conference on Machine Learning (ICML)}, 2019.

\bibitem{dalal2018safe}
Gal Dalal, Krishnamurthy Dvijotham, Matej Vecerik, Todd Hester, Cosmin
  Paduraru, and Yuval Tassa.
\newblock Safe exploration in continuous action spaces.
\newblock {\em arXiv preprint arXiv:1801.08757}, 2018.

\bibitem{z3}
Leonardo~Mendon\c{c}a de~Moura and Nikolaj Bj{\o}rner.
\newblock {Z3: An Efficient SMT Solver}.
\newblock In {\em TACAS}, pages 337--340, 2008.

\bibitem{devlin2017robustfill}
Jacob Devlin, Jonathan Uesato, Surya Bhupatiraju, Rishabh Singh, Abdel-rahman
  Mohamed, and Pushmeet Kohli.
\newblock Robustfill: Neural program learning under noisy i/o.
\newblock In {\em Proceedings of the 34th International Conference on Machine
  Learning-Volume 70}, pages 990--998. JMLR. org, 2017.

\bibitem{DuIPSSRSM18}
Tao Du, Jeevana~Priya Inala, Yewen Pu, Andrew Spielberg, Adriana Schulz,
  Daniela Rus, Armando Solar{-}Lezama, and Wojciech Matusik.
\newblock Inversecsg: automatic conversion of 3d models to {CSG} trees.
\newblock {\em {ACM} Trans. Graph.}, 37(6):213:1--213:16, 2018.

\bibitem{duan2016benchmarking}
Yan Duan, Xi~Chen, Rein Houthooft, John Schulman, and Pieter Abbeel.
\newblock Benchmarking deep reinforcement learning for continuous control.
\newblock In {\em International Conference on Machine Learning}, pages
  1329--1338, 2016.

\bibitem{duchi2010composite}
John~C Duchi, Shai Shalev-Shwartz, Yoram Singer, and Ambuj Tewari.
\newblock Composite objective mirror descent.
\newblock In {\em COLT}, pages 14--26, 2010.

\bibitem{ellis2018learning}
Kevin Ellis, Daniel Ritchie, Armando Solar-Lezama, and Josh Tenenbaum.
\newblock Learning to infer graphics programs from hand-drawn images.
\newblock In {\em Advances in Neural Information Processing Systems}, pages
  6059--6068, 2018.

\bibitem{FeserCD15}
John~K. Feser, Swarat Chaudhuri, and Isil Dillig.
\newblock Synthesizing data structure transformations from input-output
  examples.
\newblock In {\em Proceedings of the 36th {ACM} {SIGPLAN} Conference on
  Programming Language Design and Implementation, Portland, OR, USA, June
  15-17, 2015}, pages 229--239, 2015.

\bibitem{flaxman2005online}
Abraham~D Flaxman, Adam~Tauman Kalai, and H~Brendan McMahan.
\newblock Online convex optimization in the bandit setting: gradient descent
  without a gradient.
\newblock In {\em Proceedings of the sixteenth annual ACM-SIAM symposium on
  Discrete algorithms}, pages 385--394. Society for Industrial and Applied
  Mathematics, 2005.

\bibitem{gulwani-survey}
Sumit Gulwani, Oleksandr Polozov, and Rishabh Singh.
\newblock Program synthesis.
\newblock {\em Foundations and Trends in Programming Languages}, 4(1-2):1--119,
  2017.

\bibitem{henderson2018deep}
Peter Henderson, Riashat Islam, Philip Bachman, Joelle Pineau, Doina Precup,
  and David Meger.
\newblock Deep reinforcement learning that matters.
\newblock In {\em Thirty-Second AAAI Conference on Artificial Intelligence},
  2018.

\bibitem{jha2010oracle}
Susmit Jha, Sumit Gulwani, Sanjit~A Seshia, and Ashish Tiwari.
\newblock Oracle-guided component-based program synthesis.
\newblock In {\em Proceedings of the 32nd ACM/IEEE International Conference on
  Software Engineering-Volume 1}, pages 215--224. ACM, 2010.

\bibitem{kakade2003sample}
Sham~Machandranath Kakade et~al.
\newblock {\em On the sample complexity of reinforcement learning}.
\newblock PhD thesis, University of London London, England, 2003.

\bibitem{konda2000actor}
Vijay~R Konda and John~N Tsitsiklis.
\newblock Actor-critic algorithms.
\newblock In {\em Advances in neural information processing systems}, pages
  1008--1014, 2000.

\bibitem{simile}
Hoang~M. Le, Andrew Kang, Yisong Yue, and Peter Carr.
\newblock Smooth imitation learning for online sequence prediction.
\newblock In {\em International Conference on Machine Learning (ICML)}, 2016.

\bibitem{le2019batch}
Hoang~M Le, Cameron Voloshin, and Yisong Yue.
\newblock Batch policy learning under constraints.
\newblock In {\em International Conference on Machine Learning (ICML)}, 2019.

\bibitem{ddpg}
Timothy~P Lillicrap, Jonathan~J Hunt, Alexander Pritzel, Nicolas Heess, Tom
  Erez, Yuval Tassa, David Silver, and Daan Wierstra.
\newblock Continuous control with deep reinforcement learning.
\newblock {\em arXiv preprint arXiv:1509.02971}, 2015.

\bibitem{mahadevan2012sparse}
Sridhar Mahadevan and Bo~Liu.
\newblock Sparse q-learning with mirror descent.
\newblock In {\em Proceedings of the Twenty-Eighth Conference on Uncertainty in
  Artificial Intelligence}, pages 564--573. AUAI Press, 2012.

\bibitem{montgomery2016guided}
William~H Montgomery and Sergey Levine.
\newblock Guided policy search via approximate mirror descent.
\newblock In {\em Advances in Neural Information Processing Systems}, pages
  4008--4016, 2016.

\bibitem{Bayou}
Vijayaraghavan Murali, Swarat Chaudhuri, and Chris Jermaine.
\newblock Neural sketch learning for conditional program generation.
\newblock In {\em ICLR}, 2018.

\bibitem{nemirovsky1983problem}
Arkadii~Semenovich Nemirovsky and David~Borisovich Yudin.
\newblock Problem complexity and method efficiency in optimization.
\newblock 1983.

\bibitem{parisotto2016neuro}
Emilio Parisotto, Abdel-rahman Mohamed, Rishabh Singh, Lihong Li, Dengyong
  Zhou, and Pushmeet Kohli.
\newblock Neuro-symbolic program synthesis.
\newblock {\em arXiv preprint arXiv:1611.01855}, 2016.

\bibitem{flashmeta}
Oleksandr Polozov and Sumit Gulwani.
\newblock Flashmeta: a framework for inductive program synthesis.
\newblock In {\em Proceedings of the 2015 {ACM} {SIGPLAN} International
  Conference on Object-Oriented Programming, Systems, Languages, and
  Applications, {OOPSLA} 2015, part of {SPLASH} 2015, Pittsburgh, PA, USA,
  October 25-30, 2015}, pages 107--126, 2015.

\bibitem{RaychevBVK16}
Veselin Raychev, Pavol Bielik, Martin~T. Vechev, and Andreas Krause.
\newblock Learning programs from noisy data.
\newblock In {\em Proceedings of the 43rd Annual {ACM} {SIGPLAN-SIGACT}
  Symposium on Principles of Programming Languages, {POPL} 2016, St.
  Petersburg, FL, USA, January 20 - 22, 2016}, pages 761--774, 2016.

\bibitem{ross2014reinforcement}
Stephane Ross and J~Andrew Bagnell.
\newblock Reinforcement and imitation learning via interactive no-regret
  learning.
\newblock {\em arXiv preprint arXiv:1406.5979}, 2014.

\bibitem{ross2011reduction}
St{\'e}phane Ross, Geoffrey Gordon, and Drew Bagnell.
\newblock A reduction of imitation learning and structured prediction to
  no-regret online learning.
\newblock In {\em Proceedings of the fourteenth international conference on
  artificial intelligence and statistics}, pages 627--635, 2011.

\bibitem{dagger}
St{\'{e}}phane Ross, Geoffrey~J. Gordon, and Drew Bagnell.
\newblock A reduction of imitation learning and structured prediction to
  no-regret online learning.
\newblock In {\em Proceedings of the Fourteenth International Conference on
  Artificial Intelligence and Statistics, {AISTATS} 2011, Fort Lauderdale, USA,
  April 11-13, 2011}, pages 627--635, 2011.

\bibitem{schulman2015trust}
John Schulman, Sergey Levine, Pieter Abbeel, Michael Jordan, and Philipp
  Moritz.
\newblock Trust region policy optimization.
\newblock In {\em International Conference on Machine Learning}, pages
  1889--1897, 2015.

\bibitem{schulman2017proximal}
John Schulman, Filip Wolski, Prafulla Dhariwal, Alec Radford, and Oleg Klimov.
\newblock Proximal policy optimization algorithms.
\newblock {\em arXiv preprint arXiv:1707.06347}, 2017.

\bibitem{SiYDNS19}
Xujie Si, Yuan Yang, Hanjun Dai, Mayur Naik, and Le~Song.
\newblock Learning a meta-solver for syntax-guided program synthesis.
\newblock In {\em 7th International Conference on Learning Representations,
  {ICLR} 2019, New Orleans, LA, USA, May 6-9, 2019}, 2019.

\bibitem{armandosketching}
Armando Solar-Lezama, Liviu Tancau, Rastislav Bod\'{\i}k, Sanjit~A. Seshia, and
  Vijay~A. Saraswat.
\newblock Combinatorial sketching for finite programs.
\newblock In {\em ASPLOS}, pages 404--415, 2006.

\bibitem{sun2018truncated}
Wen Sun, J~Andrew Bagnell, and Byron Boots.
\newblock Truncated horizon policy search: Combining reinforcement learning \&
  imitation learning.
\newblock In {\em International Conference on Learning Representations (ICLR)},
  2018.

\bibitem{sun2018dual}
Wen Sun, Geoffrey~J Gordon, Byron Boots, and J~Bagnell.
\newblock Dual policy iteration.
\newblock In {\em Advances in Neural Information Processing Systems}, pages
  7059--7069, 2018.

\bibitem{sun2017deeply}
Wen Sun, Arun Venkatraman, Geoffrey~J Gordon, Byron Boots, and J~Andrew
  Bagnell.
\newblock Deeply aggrevated: Differentiable imitation learning for sequential
  prediction.
\newblock In {\em International Conference on Machine Learning (ICML)}, 2017.

\bibitem{sutton2018reinforcement}
Richard~S Sutton and Andrew~G Barto.
\newblock {\em Reinforcement learning: An introduction}.
\newblock MIT press, 2018.

\bibitem{sutton2000policy}
Richard~S Sutton, David~A McAllester, Satinder~P Singh, and Yishay Mansour.
\newblock Policy gradient methods for reinforcement learning with function
  approximation.
\newblock In {\em Advances in neural information processing systems}, pages
  1057--1063, 2000.

\bibitem{thomas2013projected}
Philip~S Thomas, William~C Dabney, Stephen Giguere, and Sridhar Mahadevan.
\newblock Projected natural actor-critic.
\newblock In {\em Advances in neural information processing systems}, pages
  2337--2345, 2013.

\bibitem{valkov2018houdini}
Lazar Valkov, Dipak Chaudhari, Akash Srivastava, Charles Sutton, and Swarat
  Chaudhuri.
\newblock Houdini: Lifelong learning as program synthesis.
\newblock In {\em Advances in Neural Information Processing Systems}, pages
  8687--8698, 2018.

\bibitem{pirl}
Abhinav Verma, Vijayaraghavan Murali, Rishabh Singh, Pushmeet Kohli, and Swarat
  Chaudhuri.
\newblock Programmatically interpretable reinforcement learning.
\newblock In {\em International Conference on Machine Learning}, pages
  5052--5061, 2018.

\bibitem{TORCS}
Bernhard Wymann, Eric Espi{\'e}, Christophe Guionneau, Christos Dimitrakakis,
  R{\'e}mi Coulom, and Andrew Sumner.
\newblock {TORCS}, {T}he {O}pen {R}acing {C}ar {S}imulator.
\newblock http://www.torcs.org, 2014.

\bibitem{zhu2019an}
He~Zhu, Zikang Xiong, Stephen Magill, and Suresh Jagannathan.
\newblock An inductive synthesis framework for verifiable reinforcement
  learning.
\newblock In {\em ACM Conference on Programming Language Design and
  Implementation (SIGPLAN)}, 2019.

\end{thebibliography}

\newpage
\appendix
\title{Appendix: Imitation-Projected Programmatic Reinforcement Learning}
\maketitle
\section{Theoretical Analysis}
\subsection{Preliminaries and Notations}
\label{sec:app_theory_preliminaries}
We formally define an ambient control policy space $\mathcal{U}$ to be a vector space equipped with inner product $\inner{\cdot,\cdot}:\mathcal{U}\times\mathcal{U}\mapsto\mathbb{R}$, which induces a norm $\norm{u} = \sqrt{\inner{u,u}}$, and its dual norm defined as $\norm{v}_* = \sup\{ \inner{v,u}| \norm{u}\leq 1\}$. While multiple ways to define the inner product exist, for concreteness we can think of the example of square-integrable stationary policies with $\inner{u,v} = \int_{\mathcal{S}} u(s) v(s) ds$. The addition operator $+$ between two policies $u,v\in\mathcal{U}$ is defined as $(u+v)(s) = u(s)+v(s)$ for all state $s\in\mathcal{S}$. Scaling $\lambda u + \kappa v$ is defined similarly for scalar $\lambda,\kappa$. 

The cost functional of a control policy $u$ is defined as $J(u) = \int_0^\infty c(s(\tau),u(\tau)) d\tau$, or $J(u) = \int_\mathcal{S} c(s,u(s)) d\mu^u(s)$, where $\mu^u$ is the distribution of states induced by policy $u$. This latter example is equivalent to the standard notion of value function in reinforcement learning. 

Separate from the parametric representation issues, both programmatic policy class $\Pi$ and neural policy class $\F$, and by extension - the joint policy class $\mathcal{H}$, are considered to live in the ambient vector space $\mathcal{U}$. We thus have a common and well-defined notion of distance between policies from different classes. 

We make an important distinction between differentiability of $J(h)$ in the ambient policy space (non-parametric), versus differentiability in parameterization (parametric). For example, if $\Pi$ is a class of decision-tree based policy, policies in $\Pi$ may not be differentiable under representation. However, policies $\pi\in\Pi$ might still be differentiable when considered as points in the ambient vector space $\mathcal{U}$. 

We will use the following standard notion of gradient and differentiability from functional analysis:
\begin{defn}[Subgradients]
The subgradient of $J$ at $h$, denoted $\partial J(h)$, is the non-empty set $\{ g \in\HH| \forall j\in\HH: \inner{j-h, g} + J(h) \leq J(j)\}$
\end{defn}

\begin{defn}[Fr\'echet gradient]
A bounded linear operator $\nabla:\mathcal{H}\mapsto\mathcal{H}$ is called Fr\'echet functional gradient of $J$ at $h\in\HH$ if $\lim\limits_{\norm{g}\rightarrow 0} \frac{J(h+g)-J(h) - \inner{\nabla J(h),g}}{\norm{g}} = 0$
\end{defn}
The notions of convexity, smoothness and Bregman divergence are analogous to finite-dimensional setting: 
\begin{defn}[Strong convexity]
A differentiable function $R$ is $\alpha-$strongly convex w.r.t norm $\norm{\cdot}$ if $R(y)\geq R(x) +\inner{\nabla R(x),y-x}+\frac{\alpha}{2}\norm{y-x}^2$
\end{defn}
\begin{defn}[Lipschitz continuous gradient smoothness]
A differentiable function $R$ is $L_R-$strongly smooth w.r.t norm $\norm{\cdot}$ if $\norm{\nabla R(x) - \nabla R(y)}_* \leq L_R\norm{x-y}$ 
\end{defn}
\begin{defn}[Bregman Divergence]
For a strongly convex regularizer $R$, $D_R(x,y) = R(x) - R(y) - \inner{\nabla R(y), x-y}$ is the Bregman divergence between $x$ and $y$ (not necessarily symmetric)
\end{defn}
The following standard result for Bregman divergence will be useful:
\begin{lem}\cite{beck2003mirror}
\label{lem:app_theory_1}
For all $x,y,z$ we have the identity $\inner{\nabla R(x) - \nabla R(y), x-z} = D_R(x,y) + D_R(z,x) - D_R(z,y)$. Since Bregman divergence is non-negative, a consequence of this identity is that $D_R(z,x) - D_R(z,y) \leq \inner{\nabla R(x) - \nabla R(y), z-x}$
\end{lem}
\subsection{Expected Regret Bound under Noisy Policy Gradient Estimates and Projection Errors}
\label{sec:app_regret_bound}
In this section, we show regret bound for the performance of the sequence of returned programs $\pi_1,\ldots, \pi_T$ of the algorithm. The analysis here is agnostic to the particular implementation of algorithm \ref{alg:update} and algorithm \ref{alg:project}.

Let $R$ be a $\alpha-$strongly convex and $L_R-$smooth functional with respect to norm $\norm{\cdot}$ on $\HH$. The steps from algorithm \ref{alg:ippg} can be described as follows. 
\begin{itemize}
    \item Initialize $\pi_0\in\Pi$. For each iteration $t$:
    \begin{enumerate}
        \item Obtain a noisy estimate of the gradient $\widehat{\nabla} J(\pi_{t-1}) \approx \nabla J(\pi_{t-1})$
        \item Update in the $\HH$ space: $\nabla R(h_{t}) = \nabla R(\pi_{t-1}) - \eta \widehat{\nabla} J(\pi_{t-1})$
        \item Obtain approximate projection $\pi_t = \project_\pi^R(h_t) \approx \argmin_{pi\in\Pi} D_R(\pi,h_t)$
    \end{enumerate}
\end{itemize}
This procedure is an approximate functional mirror descent scheme under bandit feedback. We will develop the following result, which is a more detailed version of \ref{thm:expected_regret_bound} in the main paper.

In the statement below, $D$ is the diameter on $\Pi$ with respect to defined norm $\norm{\cdot}$ (i.e., $D = \sup\norm{\pi - \pi^\prime}$). $L_J$ is the Lipschitz constant of the functional $J$ on $\HH$. $\beta, \sigma^2$ are the bound on the bias and variance of the gradient estimate at each iteration, respectively. $\alpha$ and $\L_R$ are the strongly convex and smooth coefficients of the functional regularizer $R$. Finally, $\epsilon$ is the bound on the projection error with respect to the same norm $\norm{\cdot}$.
\begin{thm}[Regret bound of returned policies]
\label{thm:app_expected_regret_bound}
Let $\pi_1,\ldots, \pi_T$ be a sequence of programmatic policies returned by algorithm \ref{alg:ippg} and $\pi^*$ be the optimal programmatic policy. We have the expected regret bound:
    $$\mathbb{E}\left[ \frac{1}{T}\sum_{t=1}^T J(\pi_t)\right] - J(\pi^*) \leq \frac{L_R D^2}{\eta T} + \frac{\epsilon L_R D}{\eta} + \frac{\eta(\sigma^2 + L_J^2)}{\alpha} + \beta D$$
In particular, choosing the learning rate $\eta = \sqrt{\frac{\frac{1}{T}+\epsilon}{\sigma^2}}$, the expected regret is simplified into: 
\begin{equation}
    \label{eqn:app_expected_regret_bound}
\end{equation}$$\mathbb{E}\left[ \frac{1}{T}\sum_{t=1}^T J(\pi_t)\right] - J(\pi^*) = O\left( \sigma\sqrt{\frac{1}{T}+\epsilon} + \beta\right)$$
\end{thm} 
\vspace{-0.15in}
\begin{proof}
At each round $t$, let $\widebar{\nabla}_{t} = \E[\widehat{\nabla}_t | \pi_t]$ be the conditional expectation of the gradient estimate. We will use the shorthand notation $\nabla_t = \nabla J(\pi_t)$. Denote the upper-bound on the bias of the estimate by $\beta_t$, i.e., $\norm{\widebar{\nabla}_t-\nabla_t}_* \leq \beta_t$ almost surely. Denote the noise of the gradient estimate by $\xi_t = \widebar{\nabla}_t - \widehat{\nabla}_t$, and $\sigma_t^2 = \E\big[ \norm{\widehat{\nabla}_t - \widebar{\nabla}_t}_*^2\big]$ is the variance of gradient estimate $\widehat{\nabla}_t$. 

The projection operator is $\epsilon-$approximate in the sense that $\norm{\pi_t - \project_\Pi^R(f_t)} = \norm{\widehat{\project}_\Pi^R(h_t) - \project_\Pi^R(h_t)} \leq \epsilon$ with some constant $\epsilon$, which reflects the statistical error of the imitation learning procedure. This projection error in general is independent of the choice of function classes $\Pi$ and $\F$.We will use the shorthand notation $\pi_t^* = \project_\Pi^R(f_t)$ for the true Bregman projection of $h_t$ onto $\Pi$.

Due to convexity of $J$ over the space $\HH$ (which includes $\Pi$), we have for all $\pi\in\Pi$:
$$J(\pi_t) - J(\pi) \leq \inner{\nabla_t, \pi_t-\pi}$$
We proceed to bound the RHS, starting with bounding the inner product where the actual gradient is replaced by the estimated gradient.
\begin{align}
    &\inner{\widehat{\nabla}_t, \pi_t - \pi} = \frac{1}{\eta_t}\inner{\nabla R(\pi_t) - \nabla R(h_{t+1}), \pi_t - \pi} \label{eqn:app_theory_1}\\ 
    &= \frac{1}{\eta_t}\big( D_R(\pi,\pi_t) - D_R(\pi,h_{t+1}) + D_R(\pi_t,h_{t+1})\big) \label{eqn:app_theory_2}\\
    &\leq \frac{1}{\eta_t} \big( D_R(\pi,\pi_t)- D_R(\pi,\pi_{t+1}^*) - D_R(\pi_{t+1}^*, h_{t+1}) + D_R(\pi_t,h_{t+1}) \big) \label{eqn:app_theory_3} \\
    &= \frac{1}{\eta_t} \big( \underbrace{D_R(\pi,\pi_t) - D_R(\pi,\pi_{t+1})}_{\text{telescoping}} + \underbrace{D_R(\pi,\pi_{t+1}) - D_R(\pi,\pi_{t+1}^*)}_{\text{projection error}} \underbrace{- D_R(\pi_{t+1}^*,h_{t+1}) + D_R(\pi_t,h_{t+1})}_{\text{relative improvement}}\big) \label{eqn:app_theory_4a}
\end{align}
Equation (\ref{eqn:app_theory_1}) is due to the gradient update rule in $\F$ space. Equation (\ref{eqn:app_theory_2}) is derived from definition of Bregman divergence. Equation (\ref{eqn:app_theory_3}) is due to the generalized Pythagorean theorem of Bregman projection $D_R(x,y) \geq D_R(x, \project_\Pi^R(x)) + D_R(\project_\Pi^R(x),y)$. The RHS of equation (\ref{eqn:app_theory_3}) are decomposed into three components that will be bounded separately. 

\emph{Bounding projection error.} By lemma (\ref{lem:app_theory_1}) we have
\begin{equation}
D_R(\pi,\pi_{t+1}) - D_R(\pi,\pi_{t+1}^*) \leq \inner{\nabla R(\pi_{t+1}) - \nabla R(\pi_{t+1}^*), \pi-\pi_{t+1}} \label{eqn:app_theory_4}    
\end{equation}
\begin{equation}
    \qquad\qquad\leq \norm{\nabla R(\pi_{t+1}) - \nabla R(\pi_{t+1}^*)} \norm{\pi-\pi_{t+1}}_* \label{eqn:app_theory_5}
\end{equation}
\begin{equation}
    \qquad\qquad\leq L_R\norm{\pi_{t+1} -\pi_{t+1}^*}D\leq \epsilon L_R D \label{eqn:app_theory_6}
\end{equation}
Equation (\ref{eqn:app_theory_5}) is due to Cauchy–Schwarz. Equation (\ref{eqn:app_theory_6}) is due to Lipschitz smoothness of $\nabla R$ and definition of $\epsilon-$approximate projection. 

\emph{Bounding relative improvement.} This follows standard argument from analysis of mirror descent algorithm.
\begin{align}
    &D_R(\pi_t,h_{t+1}) - D_R(\pi_{t+1}^*,h_{t+1}) = R(\pi_t) - R(\pi_{t+1}^*) + \inner{\nabla R(h_{t+1}),\pi_{t+1}^* - \pi_t} \\ 
    &\leq \inner{\nabla R(\pi_t), \pi_t - \pi_{t+1}^*} - \frac{\alpha}{2}\norm{\pi_{t+1}^*-\pi_t}_*^2 + \inner{\nabla R(h_{t+1}), \pi_{t+1}^* - \pi_t} \label{eqn:app_theory_7} \\
    &= - \eta_t\inner{\widehat{\nabla}_t,\pi_{t+1}^* - \pi_t} - \frac{\alpha}{2}\norm{\pi_{t+1}^* - \pi_t}^2 \label{eqn:app_theory_8} \\
    &\leq \frac{\eta_t^2}{2\alpha}\norm{\widehat{\nabla}_t}_*^2\leq \frac{\eta_t^2}{\alpha}(\sigma_t^2+L_J^2)
\end{align}
Equation (\ref{eqn:app_theory_7}) is from the $\alpha-$strong convexity property of regularizer $R$. Equation (\ref{eqn:app_theory_8}) is by definition of the gradient update. 
Combining the bounds on the three components and taking expectation, we thus have 
\begin{equation}
     \mathbb{E} \left[ \inner{\widehat{\nabla}_t, \pi_t-\pi}\right] \leq \frac{1}{\eta_t} \left( D_R(\pi,\pi_t) - D_R(\pi,\pi_{t+1}) +\epsilon L_R D+\frac{\eta_t^2}{\alpha}(\sigma_t^2+L_J^2) \right) \label{eqn:app_theory_12}
\end{equation}
Next, the difference between estimated gradient $\widehat{\nabla}_t$ and actual gradient $\nabla_t$ factors into the bound via Cauchy-Schwarz:
\begin{equation}
     \mathbb{E}\left[\inner{\nabla_t - \widehat{\nabla}_t,\pi_t - \pi}\right] \leq  \norm{\nabla_t -\mathbb{E}[\widehat{\nabla}_t] }_* \norm{\pi_t-\pi}\leq \beta_t D \label{eqn:app_theory_13}
\end{equation}
The results can be deduced from equations (\ref{eqn:app_theory_12}) and (\ref{eqn:app_theory_13}). 

\textbf{Unbiased gradient estimates.} For the case when the gradient estimate is unbiased, assume the variance of the noise of gradient estimates is bounded by $\sigma^2$, we have the expected regret bound for all $pi\in\Pi$
\begin{equation}
    \mathbb{E} \left[\frac{1}{T}\sum_{t=1}^T J(\pi_t)\right] - J(\pi) \leq \frac{L_R D^2}{\eta T}+\frac{\epsilon L_R D}{\eta} +\frac{\eta(\sigma^2 + L_J^2)}{\alpha} \label{eqn:app_theory_14}
\end{equation}
here to clarify, $L_R$ is the smoothness coefficient of regularizer $R$ (i.e., the gradient of $R$ is $L_R$-Lipschitz, $L_J$ is Lipschitz constant of $J$, $D$ is the diameter of $\Pi$ under norm $\norm{\cdot}$, $\sigma^2$ is the upper-bound on the variance of gradient estimates, and $\epsilon$ is the error from the projection procedure (i.e., imitation learning loss).

We can set learning rate $\eta = \sqrt{\frac{\frac{1}{T}+\epsilon}{\sigma^2}}$ to observe that the expected regret is bounded by $O(\sigma\sqrt{\frac{1}{T} + \epsilon})$. 

\textbf{Biased gradient estimates.} Assume that the bias of gradient estimate at each round is upper-bounded by $\beta_t\leq \beta$. Similar to before, combining inequalities from (\ref{eqn:app_theory_12}) and (\ref{eqn:app_theory_13}), we have 
\begin{equation}
    \mathbb{E} \left[\frac{1}{T}\sum_{t=1}^T J(\pi_t)\right] - J(\pi) \leq \frac{L_R D^2}{\eta T}+\frac{\epsilon L_R D}{\eta} +\frac{\eta(\sigma^2 + L_J^2)}{\alpha} + \beta D \label{eqn:app_theory_15}
\end{equation}
Similar to before, we can set learning rate $\eta = \sqrt{\frac{\frac{1}{T}+\epsilon}{\sigma^2}}$ to observe that on the expected regret is bounded by $O(\sigma\sqrt{\frac{1}{T} + \epsilon} +\beta )$. 
Compared to the bound on (\ref{eqn:app_theory_14}), in the biased case, the extra regret incurred per bound is simply a constant, and does not depend on $T$. 
\end{proof}

\subsection{Finite-Sample Analysis}
\label{sec:app_finite_sample}
In this section, we provide overall finite-sample analysis for \ippg under some simplifying assumptions. We first consider the case where exact gradient estimate is available, before extending the result to the general case of noisy policy gradient update. Combining the two steps will give us the proof for the following statement (theorem \ref{thm:finite_sample_main_paper} in the main paper)
\begin{thm}[Finite-sample guarantee]
\label{thm:finite_sample_app}
At each iteration, we perform vanilla policy gradient estimate of $\pi$ (over $\HH$) using $m$ trajectories and use DAgger algorithm to collect $M$ roll-outs. Setting the learning rate $\eta = \sqrt{\frac{1}{\sigma^2}\big(\frac{1}{T} +\frac{H}{M} + \sqrt{\frac{\log(T/\delta)}{M}}\big)}$, after $T$ rounds of the algorithm, we have that
\begin{small}
$$\frac{1}{T}\sum_{t=1}^T J(\pi_t) - J(\pi^*) \leq O\left(\sigma\sqrt{\frac{1}{T} + \frac{H}{M} +\sqrt{\frac{\log(T/\delta)}{M}}}\right) + O\left(\sigma\sqrt{\frac{\log(Tk/\delta)}{m}}+\frac{AH\log(Tk/\delta)}{m}\right) $$
\end{small}
\hspace{-3pt}holds with probability at least $1-\delta$, with $H$ the task horizon, $A$ the cardinality of action space, $\sigma^2$ the variance of policy gradient estimates, and $k$ the dimension $\Pi$'s parameterization.
\end{thm}

\textbf{Exact gradient estimate case.} Assuming that the policy gradients can be calculated exactly, it is straight-forward to provide high-probability guarantee for the effect of the projection error. We start with the following result, adapted from \cite{ross2011reduction} for the case of projection error bound. In this version of DAgger, we assume that we only collect a single \emph{(state, expert action)} pair from each trajectory roll-out. Result is similar, with tighter bound, when multiple data points are collected along the trajectory. 
\begin{lem}[Projection error bound from imitation learning procedure]
\label{lem:app_dagger} 
Using DAgger as the imitation learning sub-routine for our \project operator in algorithm \ref{alg:project}, let $M$ be the number of trajectories rolled-out for learning, and $H$ be the horizon of the task. With probability at least $1-\delta$, we have
$$D_R(\pi,\pi^*) \leq \widetilde{O}(1/M) + \frac{2 \ell_{max} (1+H)}{M} + \sqrt{\frac{2 \ell_{max} \log(1/\delta))}{M}}$$
where $\pi$ is the result of \project, $\pi^*$ is the true Bregman projection of $h$ onto $\Pi$, and $\ell_{max}$ is the maximum value of the imitation learning loss function $D_R(\cdot,\cdot)$
\end{lem}
The bound in lemma \ref{lem:app_dagger} is simpler than previous imitation learning results with cost information (\cite{ross2014reinforcement, ross2011reduction}. The reason is that the goal of the \project operator is more modest. Since we only care about the distance between the empirical projection $\pi$ and the true projection $\pi^*$, the loss objective in imitation learning is simplified (i.e., this is only a regret bound), and we can disregard how well policies in $\Pi$ can imitate the expert $h$, as well as the performance of $J(\pi)$ relative to the true cost from the environment $J(h)$. 

A consequence of this lemma is that for the number of trajectories at each round of imitation learning $M = O(\frac{\log 1/\delta}{\epsilon^2}) + O(\frac{H}{\epsilon})$, we have $D_R(\pi_t,\pi_t^*)\leq \epsilon$ with probability at least $1-\delta$. Applying union bound across $T$ rounds of learning, we obtain the following guarantee (under no gradient estimation error)
\begin{prop}[Finite-sample Projection Error Bound]
\label{prop:finite_sample_proj_error}
To simplify the presentation of the result, we consider $L_R, D, L, \alpha$ to be known constants. Using DAgger algorithm to collect $M = O(\frac{\log T/\delta}{\epsilon^2}) + O(\frac{H}{\epsilon})$ roll-outs at each iteration, we have the following regret guarantee after $T$ rounds of our main algorithm:
$$\frac{1}{T}\sum_{t=1}^T J(\pi_t) - J(\pi^*) \leq O\left(\frac{1}{\eta T} + \frac{\epsilon}{\eta} + \eta\right)$$
with probability at least $1-\delta$. Consequently, setting $\eta = \sqrt{\frac{1}{T} +\frac{H}{M} + \sqrt{\frac{\log(T/\delta)}{M}}}$, we have that 
$$\frac{1}{T}\sum_{t=1}^T J(\pi_t) - J(\pi^*) \leq O\left(\sqrt{\frac{1}{T} + \frac{H}{M} +\sqrt{\frac{\log(T/\delta)}{M}}}\right) $$ with probability at least $1-\delta$
\end{prop}
Note that the dependence on the time horizon of the task is sub-linear. This is different from standard imitation learning regret bounds, which are often at least linear in the task horizon. The main reason is that our comparison benchmark $\pi^*$ does live in the space $\Pi$, whereas for DAgger, the expert policy may not reside in the same space. 

\textbf{Noisy gradient estimate case.} We now turn to the issue of estimating the gradient of $\nabla J(\pi)$. We make the following simplifying assumption about the gradient estimation: 
\begin{itemize}
    \item The $\pi$ is parameterized by vector $\theta\in\mathbb{R}^k$ (such as a neural network). The parameterization is differentiable with respect to $\theta$ (Alternatively, we can view $\Pi$ as a differentiable subspace of $\F$, in which case we have $\HH = \F$)
    \item At each \update loop, the policy is rolled out $m$ times to collect the data, each trajectory has horizon length $H$
    \item For each visited state $s \sim d_h$, the policy takes a uniformly random action $a$. The action space is finite with cardinality $A$.
    \item The gradient $\nabla h_\theta$ is bounded by $B$ 
\end{itemize}
The gradient estimate is performed consistent with a generic policy gradient scheme, i.e.,
$$\widehat{\nabla} J(\theta) = \frac{A}{m}\sum_{i=1}^H\sum_{j=1}^m \nabla \pi_\theta(a_i^j | s_i^j,\theta) \widehat{Q}_i^j$$ where $\widehat{Q}_i^j$ is the estimated cost-to-go \cite{sutton2000policy}. 

Taking uniform random exploratory actions ensures that the samples are i.i.d. We can thus apply Bernstein's inequality to obtain the bound between estimated gradient and the true gradient. Indeed, with probability at least $1-\delta$, we have that the following bound on the bias component-wise:
$$\norm{\widehat{\nabla} J(\theta) - \nabla J(\theta)}_\infty \leq \beta \text{ when }m\geq \frac{(2\sigma^2 + 2AHB\frac{\beta}{3})\log\frac{k}{\delta}}{\beta^2}$$
which leads to similar bound with respect to $\norm{\cdot}_*$ (here we leverage the equivalence of norms in finite dimensional setting):
$$\norm{\nabla_t - \widehat{\nabla}_t}_* \leq \beta \text{ when } m = O\left(\frac{(\sigma^2+AHB\beta)\log\frac{k}{\delta}}{\beta^2}\right)$$
Applying union bound of this result over $T$ rounds of learning, and combining with the result from proposition (\ref{prop:finite_sample_proj_error}), we have the following finite-sample guarantee in the simplifying policy gradient update. This is also the more detailed statement of theorem \ref{thm:finite_sample_main_paper} in the main paper.
\begin{prop}[Finite-sample Guarantee under Noisy Gradient Updates and Projection Error] 
\label{prop:finite_sample_overall}
At each iteration, we perform policy gradient estimate using $m=O(\frac{(\sigma^2+AHB\beta)\log\frac{Tk}{\delta}}{\beta^2})$ trajectories and use DAgger algorithm to collect $M = O(\frac{\log T/\delta}{\epsilon^2}) + O(\frac{H}{\epsilon})$ roll-outs. Setting the learning rate $\eta = \sqrt{\frac{1}{\sigma^2}\big(\frac{1}{T} +\frac{H}{M} + \sqrt{\frac{\log(T/\delta)}{M}}\big)}$, after $T$ rounds of the algorithm, we have that
$$\frac{1}{T}\sum_{t=1}^T J(\pi_t) - J(\pi^*) \leq O\left(\sigma\sqrt{\frac{1}{T} + \frac{H}{M} +\sqrt{\frac{\log(T/\delta)}{M}}}\right) + \beta $$ with probability at least $1-\delta$.

Consequently, we also have the following regret bound:
$$\frac{1}{T}\sum_{t=1}^T J(\pi_t) - J(\pi^*) \leq O\left(\sigma\sqrt{\frac{1}{T} + \frac{H}{M} +\sqrt{\frac{\log(T/\delta)}{M}}}\right) + O\left(\sigma\sqrt{\frac{\log(Tk/\delta)}{m}}+\frac{AH\log(Tk/\delta)}{m}\right) $$
holds with probability at least $1-\delta$, where again $H$ is the task horizon, $A$ is the cardinality of action space, and $k$ is the dimension of function class $\Pi$'s parameterization.
\end{prop}
\begin{proof}
(For both proposition (\ref{prop:finite_sample_overall}) and (\ref{prop:finite_sample_proj_error})).
The results follow by taking the inequality from equation (\ref{eqn:app_theory_15}), and by solving for $\epsilon$ and $\beta$ explicitly in terms of relevant quantities. Based on the specification of $M$ and $m$, we obtain the necessary precision for each round of learning in terms of number of trajectories:
\begin{align*}
    \beta &= O(\sigma\frac{\log(k/\delta)}{m}+ \frac{AHB\log(k/\delta)}{m}) \\
    \epsilon &= O(\frac{H}{M} + \sqrt{\frac{\log(1/\delta)}{M}})
\end{align*}
Setting the learning rate $\eta = \sqrt{\frac{1}{\sigma^2}\big(\frac{1}{T}+\epsilon \big)}$ and rearranging the inequalities lead to the desired bounds. 
\end{proof}
The regret bound depends on the variance $\sigma^2$ of the policy gradient estimates. It is well-known that vanilla policy gradient updates suffer from high variance. We instead use functional regularization technique, based on CORE-RL, in the practical implementation of our algorithm. The CORE-RL subroutine has been demonstrated to reduce the variance in policy gradient updates \cite{corerl}.
\subsection{Defining a consistent approximation of $\nabla_\HH J(\pi)$ - Proof of Proposition \ref{prop:continuity_gradient}}
\label{sec:app_define_gradient}
We are using the notion of Fr\'echet derivative to define gradient of differentiable functional. Note that while Gateaux derivative can also be utilized, Fr\'echet derivative ensures continuity of the gradient operator that would be useful for our analysis. 
\begin{defn}[Fr\'echet gradient]
A bounded linear operator $\nabla:\mathcal{H}\mapsto\mathcal{H}$ is called Fr\'echet functional gradient of $J$ at $h\in\HH$ if $\lim\limits_{\norm{g}\rightarrow 0} \frac{J(h+g)-J(h) - \inner{\nabla J(h),g}}{\norm{g}} = 0$
\end{defn}
We make the following assumption about $\HH$ and $\F$. One interpretation of this assumption is that the space of policies $\Pi$ and $\F$ that we consider have the property that a programmatic policy $\pi\in\Pi$ can be well-approximated by a large space of neural policies $f\in\F$. 
\begin{assumption}
$J$ is Fr\'echet differentiable on $\HH$. $J$ is also differentiable on the restricted subspace $\F$. And $\F$ is dense in $\HH$ (i.e., the closure $\widebar{\F} = \HH$)
\end{assumption}
It is then clear that $\forall$ $f\in\F$ the Fr\'echet gradient $\nabla_\F J(f)$, restricted to the subspace $\F$ is equal to the gradient of $f$ in the ambient space $\HH$ (since Fr\'echet gradient is unique). In general, given $\pi\in\Pi$ and $f\in\F$, $\pi+f$ is not necessarily in $\F$. However, the restricted gradient on subspace $\F$ of $J(\pi+f)$ can be defined asymptotically. 
\begin{prop}
Fixing a policy $\pi\in\Pi$, define a sequence of policies $f_k\in\F$, $k=1,2,\ldots$ that converges to $\pi$: $\lim_{k\rightarrow\infty}\norm{f_k-g} = 0$, we then have $\lim_{k\rightarrow\infty}\norm{\nabla_\F J(f_k) - \nabla_\HH J(\pi)}_* = 0$
\end{prop}
\begin{proof}
Since Fr\'echet derivative is a continuous linear operator, we have $\lim_{k\rightarrow\infty}\norm{\nabla_\HH J(f_k) - \nabla_\HH J(\pi)}_* = 0$. By the reasoning above, for $f\in\F$, the gradient $\nabla_\F J(f)$ defined via restriction to the space $\F$ does not change compared to $\nabla_\HH J(f)$, the gradient defined over the ambient space $\HH$. Thus we also have $\lim_{k\rightarrow\infty}\norm{\nabla_\F J(f_k) - \nabla_\HH J(\pi)}_* = 0$. By the same argument, we also have that for any given $\pi\in\Pi$ and $f\in\F$, even if $\pi+f\not\in\F$, the gradient $\nabla_\F J(\pi+f)$ with respect to the $\F$ can be approximated similarly.
\end{proof}
Note that we are not assuming $J(\pi)$ to be differentiable when restricting to the policy subspace $\Pi$. 
\subsection{Theoretical motivation for Algorithm \ref{alg:update} - Proof of Proposition \ref{prop:relaxed_objective} and \ref{prop:bias_variance_characterization}}
\label{sec:app_theoretical_motivation}
We consider the case where $\Pi$ is not differentiable by parameterization. Note that this does not preclude $J(\pi)$ for $\pi\in\Pi$ to be differentiable in the non-parametric function space. Two complications arise compared to our previous approximate mirror descent procedure. First, for each $\pi\in\Pi$, estimating the gradient $\nabla J(\pi)$ (which may not exist under certain parameterization, per section \ref{sec:closing_gap}) can become much more difficult. Second, the update rule $\nabla R(\pi) - \nabla_\F J(\pi)$ may not be in the dual space of $\F$, as in the simple case where $\Pi\subset\F$, thus making direct gradient update in the $\F$ space inappropriate.
\begin{assumption}
$J$ is convex in $\HH$. 
\end{assumption}
By convexity of $J$ in $\HH$, sub-gradients $\partial J(h)$ exists for all $h\in\HH$. In particular, $\partial J(\pi)$ exists for all $\pi\in\Pi$. Note that $\partial J(\pi)$ reflects sub-gradient of $\pi$ with respect to the ambient policy space $\HH$.

We will make use of the following equivalent perspective to mirror descent\cite{beck2003mirror}, which consists of two-step process for each iteration $t$
\begin{enumerate}
    \item Solve for $h_{t+1} = \argmin_{h\in\HH} \eta\inner{\partial J(\pi_t),h} + D_R(h, \pi_t)$
    \item Solve for $\pi_{t+1} = \argmin_{\pi\in\Pi} D_R(\pi,h_{t+1})$
\end{enumerate}
We will show how this version of the algorithm motivates our main algorithm. Consider step 1 of the main loop of \ippg , where given a fixed $\pi\in\Pi$, the optimization problem within $\HH$ is 
\begin{equation}
    (\text{OBJECTIVE\_}1) = \min_{h\in\HH} \eta \inner{\partial J(\pi), h} + D_R(h,\pi) \label{eqn:app_theory_9}
\end{equation}
Due to convexity of $\HH$ and the objective, problem $(\text{OBJECTIVE\_}1)$ is equivalent to:
\begin{align}
    (\text{OBJECTIVE\_}1) = &\min \inner{\partial J(\pi), h} \\
    &\text{ s.t. } D_R(h,\pi) \leq \tau \label{eqn:app_theory_10}
\end{align}
where $\tau$ depends on $\eta$. Since $\pi$ is fixed, this optimization problem can be relaxed by choosing $\lambda\in[0,1]$, and a set of candidate policies $h = \pi + \lambda f$, for all $f\in\F$, such that $D_R(h,\pi) \leq \tau$ is satisfied (Selection of $\lambda$ is possible with bounded spaces). Since this constraint set is potentially a restricted set compared to the space of policies satisfying inequality (\ref{eqn:app_theory_10}), the optimization problem (\ref{eqn:app_theory_9}) is relaxed into:
\begin{equation}
    (\text{OBJECTIVE\_}1) \leq (\text{OBJECTIVE\_}2) = \min_{f\in \F} \inner{\partial J(\pi), \pi + \lambda f}
\end{equation}
Due to convexity property of $J$, we have
\begin{equation}
    \inner{\partial J(\pi), \lambda f } = \inner{\partial J(\pi), \pi+\lambda f- \pi)} \leq J(\pi+\lambda f) - J(\pi)
\end{equation}
The original problem $\text{OBJECTIVE\_}1$ is thus upper bounded by:
$$\min_{h\in\HH} \eta\inner{\partial J(\pi),h )}+D_R(h,\pi) \leq \min_{f\in\F} J\big(\pi + \lambda f \big) - J(\pi) + \inner{\partial J(\pi),\pi} $$
Thus, a relaxed version of original optimization problem $\text{OBJECTIVE\_}1$ can be obtained by miniziming $J(\pi+ \lambda f)$ over $f\in\F$ (note that $\pi$ is fixed). This naturally motivates using functional regularization technique, such as CORE-RL algorithm \cite{corerl}, to update the parameters of differentiable function $f$ via policy gradient descent update:
$$f^\prime = f - \eta\lambda \nabla_\F \lambda J(\pi + \lambda f)$$
where the gradient of $J$ is taken with respect to the parameters of $f$ (neural networks). This is exactly the update step in algorithm \ref{alg:update} (also similar to iterative updte of CORE-RL algorithm), where the neural network policy is regularized by a prior controller $\pi$.

\textbf{Statement and Proof of Proposition \ref{prop:bias_variance_characterization}}
\begin{prop}[Regret bound for the relaxed optimization objective]
\label{prop:app_bias_variance}
Assuming $J(h)$ is $L$-strongly smooth over $\HH$, i.e., $\nabla_\HH J(h)$ is $L$-Lipschitz continuous, approximating $\update_\HH$ by $\update_F$ per Alg. \ref{alg:update} leads to the expected regret bound: $\mathbb{E}\left[ \frac{1}{T}\sum_{t=1}^T J(\pi_t)\right] - J(\pi^*) = O\left( \lambda\sigma\sqrt{\frac{1}{T}+\epsilon} + \lambda^2 L^2\right)$
\end{prop}
\begin{proof}
Instead of focusing on the bias of the gradient estimate $\nabla_\HH J(\pi)$, we will shift our focus on the alternative proximal formulation of mirror descent, under optimization and projection errors. In particular, at each iteration $t$, let $h_{t+1}^* = \argmin_{h\in\HH} \eta\inner{\nabla J(\pi_t), h} + D_R(h,\pi_t)$ and let the optimization error be defined as $\beta_t$ where $\nabla R(h_{t+1}) = \nabla R(h_{t+1}^*) + \beta_t$. Note here that this is different from (but related to) the notion of bias from gradient estimate of $\nabla J(\pi)$ used in theorem \ref{thm:expected_regret_bound} and theorem \ref{thm:app_expected_regret_bound}. The projection error from imitation learning procedure is defined similarly to theorem \ref{thm:expected_regret_bound}: $\pi_{t+1}^* = \argmin_{\pi\in\Pi} D_R(\pi,h_{t+1})$ is the true projection, and $\norm{\pi_{t+1} - \pi_{t+1}^*} \leq \epsilon$.

We start with similar bounding steps to the proof of theorem \ref{thm:expected_regret_bound}:
\begin{align}
    \inner{\nabla J(\pi_t),\pi_t - \pi} &= \frac{1}{\eta}\inner{\nabla R(h_{t+1}^*) - \nabla R(\pi_t), \pi_t - \pi} \nonumber \\
    &= \frac{1}{\eta} \left( \inner{\nabla R(h_{t+1}) - \nabla R(\pi_t),\pi_t - \pi } - \inner{\beta_t, \pi_t - \pi}\right) \nonumber \\
    &= \underbrace{\frac{1}{\eta}\left( D_R(\pi,\pi_t) - D_R(\pi,h_{t+1}) + D_R(\pi_t,h_{t+1})\right)}_{\text{component\_1}} + \underbrace{\frac{1}{\eta} \inner{\beta_t, \pi_t-\pi}}_{\text{component\_2}} \label{eqn:app_bias_variance_1}
\end{align}
As seen from the proof of theorem \ref{thm:app_expected_regret_bound}, $\text{component\_1}$ can be upperbounded by:
$\frac{1}{\eta} \big( \underbrace{D_R(\pi,\pi_t) - D_R(\pi,\pi_{t+1})}_{\text{telescoping}} + \underbrace{D_R(\pi,\pi_{t+1}) - D_R(\pi,\pi_{t+1}^*)}_{\text{projection error}} \underbrace{- D_R(\pi_{t+1}^*,h_{t+1}) + D_R(\pi_t,h_{t+1})}_{\text{relative improvement}}\big)$
The bound on $\text{projection error}$ is identical to theorem \ref{thm:app_expected_regret_bound}:
\begin{equation}
D_R(\pi,\pi_t) - D_R(\pi,\pi_{t+1}^*) \leq \epsilon L_R D    \label{eqn:app_bias_variance_2}
\end{equation}
The bound on $\text{relative improvement}$ is slightly different:
\begin{align}
    &D_R(\pi_t, h_{t+1}) - D_R(\pi_{t+1}^*, h_{t+1}) = R(\pi_t) - R(\pi_{t+1}^*) +\inner{\nabla R(h_{t+1}), \pi_{t+1}^* - \pi_t} \nonumber\\
    &= R(\pi_t) - R(\pi_{t+1}^* + \inner{\nabla R(h_{t+1}^*), \pi_{t+1}^* - \pi_t}) + \inner{\beta_t, \pi_{t+1}^* - \pi_t} \nonumber\\
    &\leq \inner{\nabla R(\pi_t), \pi_t - \pi_{t+1}^*} - \frac{\alpha}{2} \norm{\pi_{t+1}^* - \pi_t}^2 +\inner{\nabla R(h_{t+1}^*),\pi_{t+1}^* - \pi_t} + \inner{\beta_t, \pi_{t+1}^* - \pi_t} \nonumber \\
    &= -\eta\inner{\nabla J_{\HH}(\pi_t), \pi_{t+1}^* - \pi_t} - \frac{\alpha}{2}\norm{\pi_{t+1}^* - \pi_t}^2 +\inner{\beta_t, \pi_{t+1}^* - \pi_t} \label{eqn:app_bias_variance_3} \\
    &\leq \frac{\eta^2}{2\alpha}\norm{\nabla_\HH J(\pi_t)}_*^2 + \inner{\beta_t, \pi_{t+1}^* - \pi_t} \nonumber \\
    &\leq \frac{\eta^2}{2\alpha} L_J^2 + \inner{\beta_t, \pi_{t+1}^* - \pi_t} \label{eqn:app_bias_variance_4}
\end{align}
Note here that the gradient $\nabla_\HH J(\pi_t)$ is not the result of estimation. Combining equations (\ref{eqn:app_bias_variance_1}), (\ref{eqn:app_bias_variance_2}), (\ref{eqn:app_bias_variance_3}), (\ref{eqn:app_bias_variance_4}), we have:
\begin{equation}
\label{eqn:app_bias_variance_6}
    \inner{\nabla J(\pi_t), \pi_t - \pi} \leq \frac{1}{\eta} \big( D_R(\pi,\pi_t) - D_R(\pi,\pi_{t+1}) +\epsilon L_R D + \frac{\eta^2}{2\alpha} L_J^2 +\inner{\beta_t, \pi_{t+1}^* - \pi} \big)
\end{equation}
Next, we want to bound $\beta_t$. Choose regularizer $R$ to be $\frac{1}{2}\norm{\cdot}^2$ (consistent with the pseudocode in algorithm \ref{alg:update}). We have that:
$$h_{t+1}^* = \argmin_{h\in\HH} \eta\inner{\nabla J(\pi_t),h} + \frac{1}{2}\norm{h-\pi_t}^2$$
which is equivalent to:
$$h_{t+1}^* = \pi_t +\argmin_{f\in\F} \eta\inner{ \nabla J(\pi_t), f} + \frac{1}{2}\norm{f}^2$$
Let $f_{t+1}^* = \argmin_{f\in\F} \eta\inner{ \nabla J(\pi_t), f} + \frac{1}{2}\norm{f}^2$. Taking the gradient over $f$, we can see that $f_{t+1}^* = -\eta\nabla J(\pi_t)$. Let $f_{t+1}$ be the minimizer of $\min_{f\in\F}J(\pi_t + \lambda f)$. We then have $h_{t+1}^* = \pi_t + f_{t+1}^*$ and $h_{t+1} = \pi + \lambda f_{t+1}$. Thus $\beta_t = h_{t+1} - h_{t+1}^* = f_{t+1} -f_{t+1}^*$.

On one hand, we have $$J(\pi_t+\lambda f_{t+1}) \leq J(\pi_t + \omega f_{t+1}^*) \leq J(\pi_t) +\inner{\nabla J(\pi_t), \omega f_{t+1}^*} + \frac{L}{2}\norm{\omega f_{t+1}^*}^2$$
due to optimality of $f_{t+1}$ and strong smoothness property of $J$. On the other hand, since $J$ is convex, we also have the first-order condition:
$$J(\pi_t + \lambda f_{t+1}) \geq J(\pi_t) + \inner{\nabla J(\pi_t), \lambda f_{t+1}}$$
Combine with the inequality above, and subtract $J(\pi_t)$ from both sides, and using the relationship $f_{t+1}^* = -\eta\nabla J(\pi_t)$, we have that:
$$\inner{-\frac{1}{\eta} f_{t+1}^*,\lambda f_{t+1}} \leq \inner{-\frac{1}{\eta}f_{t+1}^*, \omega f_{t+1}^*} +\frac{L\omega^2}{2}\norm{f_{t+1}^*}^2$$
Since this is true $\forall \omega$, rearrange and choose $\omega$ such that $\frac{\omega}{\eta} - \frac{L \omega^2}{2} = -\frac{\lambda}{2\eta}$, namely $\omega = \frac{1-\sqrt{1-\lambda\eta L}}{L\eta}$, and complete the square, we can establish the bound that:
\begin{equation}
    \label{eqn:app_bias_variance_5}
    \norm{f_{t+1} - f_{t+1}^*} \leq \eta(\lambda L)^2 B
\end{equation}
for $B$ the upperbound on $\norm{f_{t+1}}$. We thus have $\norm{\beta_t} = O(\eta(\lambda L)^2)$. Plugging the result from equation \ref{eqn:app_bias_variance_5} into RHS of equation \ref{eqn:app_bias_variance_6}, we have:
\begin{equation}
\label{eqn:app_bias_variance_7}
\inner{\nabla J(\pi_t), \pi_t - \pi} \leq \frac{1}{\eta}\big( D_R(\pi,\pi_t) - D_R(\pi, \pi_{t+1}) + \epsilon L_R D + \frac{\eta^2}{2\alpha}L_J^2 \big) +\big( \eta (\lambda L)^2 B\big)   
\end{equation}
Since $J$ is convex in $\HH$, we have $J(\pi_t) - J(\pi) \leq \inner{\nabla J(\pi_t), \pi_t - \pi}$. Similar to theorem \ref{thm:expected_regret_bound}, setting $\eta = \sqrt{\frac{1}{\lambda^2\sigma^2}(\frac{1}{T}+\epsilon)}$ and taking expectation on both sides, we have:
\begin{equation}
    \mathbb{E}\left[ \frac{1}{T}\sum_{t=1}^T J(\pi_t)\right] - J(\pi^*) = O\big( \lambda\sigma\sqrt{\frac{1}{T}+\epsilon} +\lambda^2 L^2 \big)
\end{equation}
Note that unlike regret bound from theorem \ref{thm:expected_regret_bound} under general bias, variance of gradient estimate and projection error, $\sigma^2$ here explicitly refers to the bound on neural-network based policy gradient variance. The variance reduction of $\lambda\sigma$, at the expense of some bias, was also similarly noted in a recent functional regularization technique for policy gradient \cite{corerl}.
\end{proof}
\section{Additional Experimental Results and Details}\label{sec:appendix-experiments}
\subsection{TORCS}
We generate controllers for cars in \emph{The Open Racing Car Simulator} (\torcs) \cite{TORCS}. In its full generality \torcs provides a rich environment with input from up to $89$ sensors, and optionally the 3D graphic from a chosen camera angle in the race. The controllers have to decide the values of $5$ parameters during game play, which correspond to the acceleration, brake, clutch, gear and steering of the car.

Apart from the immediate challenge of driving the car on the track, controllers also have to make race-level strategy decisions, like making pit-stops for fuel. A lower level of complexity is provided in the Practice Mode setting of TORCS. In this mode all race-level strategies are removed. Currently, so far as we know, state-of-the-art \drl models are capable of racing only in Practice Mode, and this is also the environment that we use. Here we consider the input from $29$ sensors, and decide values for the acceleration, steering, and braking actions.

We chose a suite of tracks that provide varying levels of difficulty for the learning algorithms. In particular, for the tracks Ruudskogen and Alpine-2, the \ddpg agent is unable to reliably learn a policy that would complete a lap. We perform the experiments with twenty-five random seeds and report the median lap time over these twenty-five trials. However we note that the \torcs simulator is not deterministic even for a fixed random seed. Since we model the environment as a Markov Decision Process, this non-determinism is consistent with our problem statement.

For our Deep Reinforcement Learning agents we used standard open source implementations (with pre-tuned hyper-parameters) for the relevant domain.

All experiments were conducted on standard workstation with a 2.5 GHz Intel Core i7 CPU and a GTX 1080 Ti GPU card.

The code for the \torcs experiments can be found at: \href{https://bitbucket.org/averma8053/propel}{https://bitbucket.org/averma8053/propel}

In Table~\ref{table:initddpg} we show the lap time performance and crash ratios of \ippg agents initialized with neural policies obtained via \ddpg. As discussed in Section~\ref{sec:evaluation}, \ddpg often exhibits high variance across trials and this adversely affects the performance of the \ippg agents when they are initialized via \ddpg. In Table~\ref{table:treegen} we show generalization results for the \tree agent. As noted in Section~\ref{sec:evaluation}, the generalization results for \tree are in between those of \ddpg and \prog.

\textbf{Verified Smoothness Property.} For the program given in Figure~2 we proved using symbolic verification techniques, that $  \forall k, \ \sum_{i=k}^{k+5}\norm{\pickc(s[\rpm], i+1) - \pickc(s[\rpm], i)} < 0.003  \implies  \norm{\pickc(a[\mathtt{Accel}], k+1) - \pickc(a[\mathtt{Accel}], k)} < 0.63 $. Here the function $\pickc(., i)$ takes in a history/sequence of sensor or action values and returns the value at position $i$, . Intuitively, the above logical implication means that if the sum of the consecutive differences of the last six $\rpm$ sensor values is less than $0.003$, then the acceleration actions calculated at the last and penultimate step will not differ by more than $0.63$.

 \begin{table}[h]
	\caption{\textit{Performance results in \torcs of \ippg agents initialized with neural policies obtained via \ddpg, over 25 random seeds. Each entry is formatted as Lap-time / Crash-ratio, reporting median lap time in seconds over all the seeds (lower is better) and ratio of seeds that result in crashes (lower is better). A lap time of \textsc{Cr} indicates the agent crashed and could not complete a lap for more than half the seeds.}}
	\label{table:initddpg}
	\begin{center}
		\begin{small}
			\begin{sc}
				\begin{tabular}{l c c c c c}
					\toprule
					  & G-Track  & E-Road & Aalborg & Ruudskogen & Alpine-2  \\
					  Length & 3186m & 3260m & 2588m & 3274m & 3774m \\
					\midrule
					\prog-Ddpg & 97.76/.12 & 108.06/.08 & 140.48/.48 & Cr / 0.68 & Cr / 0.92 \\
					\tree-Ddpg   & 78.47/0.16 & 85.46/.04 & Cr / 0.56 & Cr / 0.68 & Cr / 0.92 \\
					\bottomrule
				\end{tabular}
			\end{sc}
		\end{small}
	\end{center}
\end{table}

\begin{table}
\caption{\textit{Generalization results in \torcs for \tree, where rows are training and columns are testing tracks. Each entry is formatted as \prog~/ DDPG, and the number reported is the median lap time in seconds over all the seeds (lower is better).  \textsc{Cr} indicates the agent crashed and could not complete a lap for more than half the seeds.}}
\label{table:treegen}
	\begin{center}
	\begin{sc}
		\begin{footnotesize}
				\begin{tabular}{lccccc}
					\toprule
					 & G-Track  & E-Road & Aalborg & Ruudskogen & Alpine-2  \\
					\midrule					
					G-Track &  -  & 95 &  Cr &  Cr & Cr    \\
					E-Road  &  84 & -  & Cr & Cr & Cr \\
					Aalborg  &  111 & Cr  & -  &  Cr &  Cr\\
					Ruudskogen  & 154 & Cr & Cr & -  & Cr \\
					Alpine-2  &  Cr & 276 &  Cr &   Cr & -  \\
					\bottomrule
				\end{tabular}
		\end{footnotesize}
		\end{sc}
	\end{center}
\end{table}
\begin{table}
	\caption{Performance results in Classic Control problems. Higher scores are better.}
	\label{table:classic}
	\begin{center}
		\begin{small}
			\begin{sc}
				\begin{tabular}{l c c}
					\toprule
					  & MountainCar & Pendulum  \\
					\midrule					
					Prior  & $00.59 \pm 0.00 $&$ -875.53 \pm 0.00$\\
					Ddpg & $97.16 \pm 3.21 $&$ -132.70 \pm 6.44$ \\
					Trpo & $93.03 \pm 1.86$ & $-131.54 \pm 4.56$ \\
					Ndps &$ 66.98 \pm 3.11$ &$ -435.71 \pm 4.83$ \\
					Viper &$ 64.86 \pm 3.28$ &$ -394.11 \pm 4.97$ \\
					\prog & $95.63 \pm 1.02 $& $-187.71  \pm 2.35$ \\
					\tree & $96.56 \pm 2.81 $ & $-139.09  \pm 3.31$ \\
					\bottomrule
				\end{tabular}
			\end{sc}
		\end{small}
	\end{center}
\end{table}

\subsection{Classic Control}
We present results from two classic control problems, Mountain-Car (with continuous actions) and Pendulum, in Table~\ref{table:classic}. We use the OpenAI Gym implementations of these environments. More information about these environments can be found at the links: \href{https://gym.openai.com/envs/MountainCarContinuous-v0/}{MountainCar} and \href{https://gym.openai.com/envs/Pendulum-v0/}{Pendulum}.

In Mountain-Car the goal is to drive an under-powered car up the side of a mountain in as few time-steps as possible. In Pendulum, the goal is to swing a pendulum up so that it stays upright. In both the environments an episode terminates after a maximum of $200$ time-steps. 

In Table~\ref{table:classic} we report the mean and standard deviation, over twenty-five random seeds, of the average scores over $100$ episodes for the listed agents and environments. In Figure~\ref{fig:mountain} and Figure~\ref{fig:pendulum} we show the improvements made over the prior by the \prog agent in MountainCar and Pendulum respectively, with each iteration of the \ippg algorithm.

\begin{figure}[h]
\centering
    \begin{minipage}{0.45\textwidth}
    \centering
    \includegraphics[width=\columnwidth]{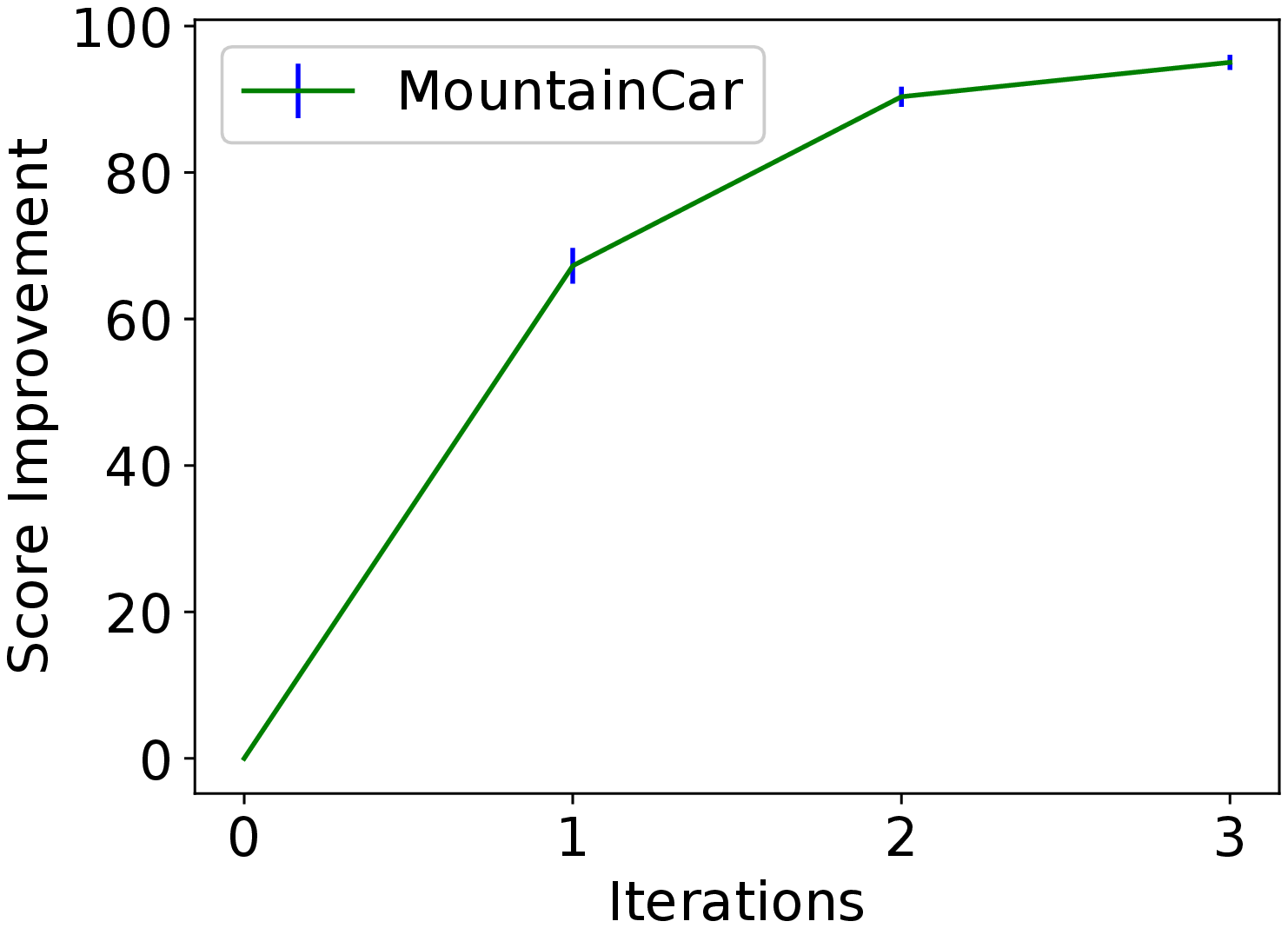}
    \caption{Score improvements in the MountainCar environment over iterations of \prog.}
  \label{fig:mountain}
 \end{minipage}\hfill
    \begin{minipage}{0.45\textwidth}
    \centering
    \includegraphics[width=\columnwidth]{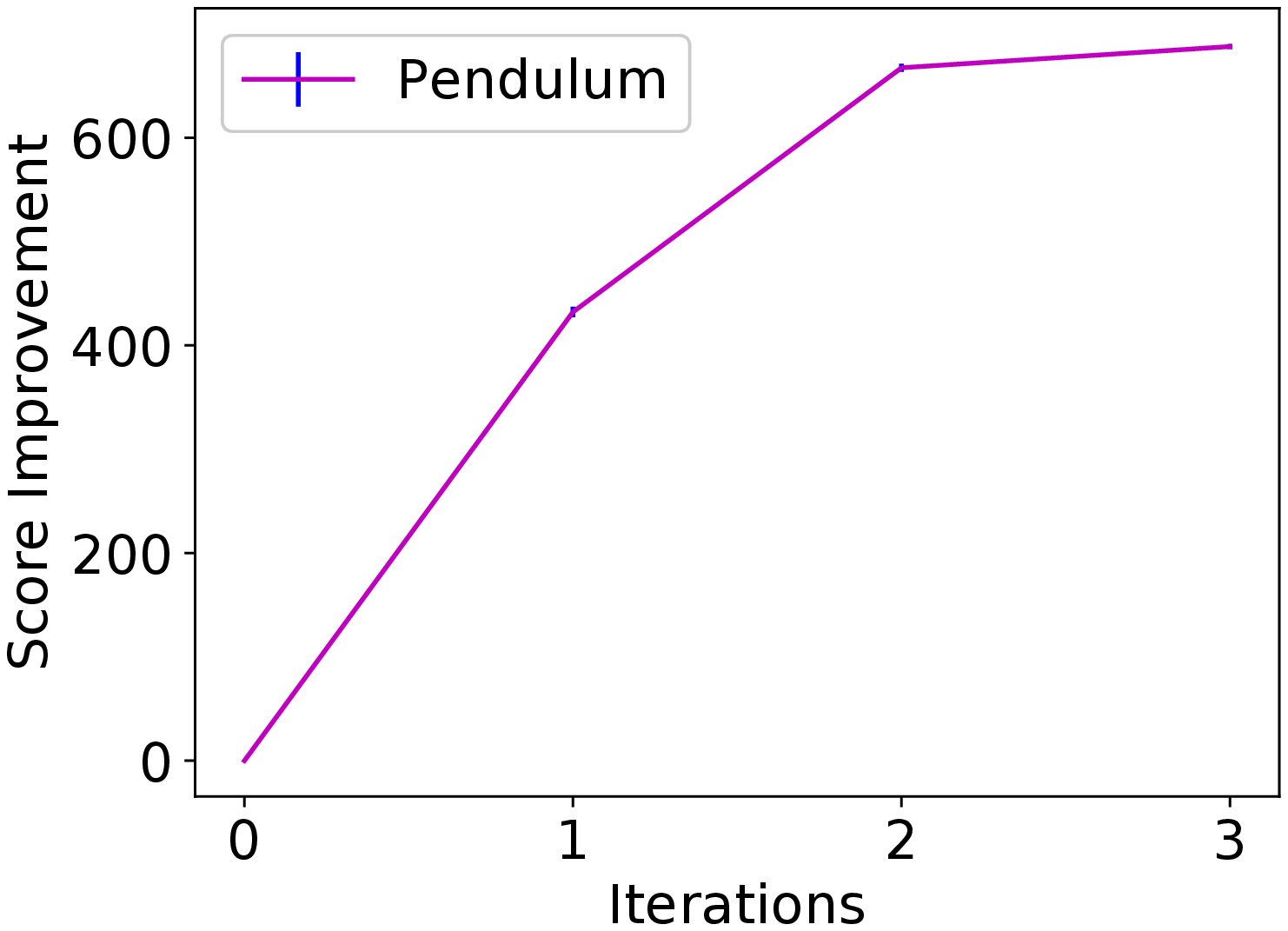}
    \caption{Score improvements in the Pendulum environment over iterations of \prog.}
  \label{fig:pendulum}
  \end{minipage}
\end{figure}

\end{document}